\newif\iftechreport % tech report version or conference version?
\techreporttrue

\iftechreport
\documentclass[a4paper]{article}
\usepackage{mythm}
\else
\documentclass[a4paper,oribibl,envcountsame]{llncs}
\fi

\usepackage[english]{babel}
\usepackage[utf8]{inputenc}

\usepackage{hyperref}
\usepackage{enumerate}
\usepackage{longtable} % for notation table in the appendix

\usepackage[flowcharts]{aixi}

\usepackage[labelfont=bf]{caption}

\AtBeginDocument{

}

\def\evi{\mathrm{evi}}
\def\pevi{\mathrm{pev}}
\def\aevi{\mathrm{aev}}
\def\cau{\mathrm{cau}}
\def\doo{{\mathtt{do}}} % the do-operator
\def\pa{{\mathit{pa}}}  % parents in a causal graph
\def\S{\mathcal{S}}     % the set of hidden states
\tikzstyle{c} = [circle, draw, text centered, minimum height=2em]
\def\eps{\varepsilon}

\def\exampledivider{\noindent\rule{\textwidth}{1pt}}
\newcommand{\appendixref}[1]{{\hyperref[#1]{Appendix~\ref*{#1}}}}

\title{Sequential Extensions of Causal and Evidential Decision Theory%
\footnote{The final publication is available at \url{http://link.springer.com/}.}}
\author{Tom Everitt \and Jan Leike \and Marcus Hutter}
\date{\today}

\begin{document}

\maketitle

\begin{abstract}%
Moving beyond the dualistic view in AI
where agent and environment are separated
incurs new challenges for decision making,
as calculation of expected utility is no longer straightforward.
The non-dualistic decision theory literature is split between
\emph{causal decision theory} and \emph{evidential decision theory}.
We extend these decision algorithms to the \emph{sequential} setting
where the agent alternates between taking actions
and observing their consequences.
We find that
evidential decision theory has two natural extensions while
causal decision theory only has one.
\end{abstract}


%%%%%%%%%%%%%%%%%%%%%%%%%%%%%%%%%%%%%%%%%%%%%%%%%%%%%%%%%%%%%%%
\section{Introduction}

%\paragraph{AI is Dualistic.}
In artificial-intelligence problems
an agent interacts sequentially with an environment
by taking actions and receiving percepts~\cite{RN:2010}.
This model is \emph{dualistic}:
the agent is distinct from the environment.
It influences the environment only through its actions,
and the environment has no other information about the agent.
The dualism assumption is accurate for
an algorithm that is playing chess, go, or other (video) games,
which explains why it is ubiquitous in AI research.
But often it is not true:
real-world agents are embedded in (and computed by) the environment~\cite{OR:2012},
and then a \emph{physicalistic model}\footnote{%
Some authors also call this type of model \emph{materialistic} or \emph{naturalistic}.
} is more appropriate.

%\paragraph{Why Physicalism?}
This distinction becomes relevant in multi-agent settings with similar agents,
where each agent encounters `echoes' of its own decision making.
If the other agents are running the same source code,
then the agents' decisions are logically connected.
This link can be used for uncoordinated cooperation ~\cite{LFYBCH:2014}.
Moreover, a physicalistic model is indispensable for self-reflection.
If the agent is required to autonomously verify its integrity,
and perform maintenance, repair, or upgrades,
then the agent needs to be aware of its own functioning.
For this, a reliable and accurate self-modeling is essential.
Today, applications of this level of autonomy are mostly restricted
to space probes distant from earth or robots navigating lethal situations,
but in the future this might also become crucial for
sustained self-improvement in generally intelligent agents%
~\cite{Yudkowsky:2008xrisk,Bostrom:2014,MIRI:2014agenda,FLI:2015}.

\begin{figure}[t]
\begin{center}
\begin{tikzpicture}[scale=0.25] % scaling everything by 1/4
% environment
\draw (0,0) -- (33,0) -- (33,12) -- (0,12) -- (0,0);
\node[above left] at (33,0) {environment};

% hidden state
\node[rectangle,draw,below left] at (32.25,11) {hidden state $s$};

% agent
\draw (0.75,1) -- (18,1) -- (18,11) -- (0.75,11) -- (0.75,1);
\draw[->] (18,6.5) to node[above] {$a_t$} (23,6.5);
\draw[<-] (18,5.5) to node[below] {$e_t$} (23,5.5);
\node[above left] at (18,1) {agent $\pi$};

% environment model
\draw (1.5,4) -- (15.5,4) -- (15.5,10) -- (1.5,10) -- (1.5,4);
\node[above left] at (15.5,4) {environment model $\mu$};

% self-model
\node[rectangle,draw,below right] at (2.25,9) {self-model};
\end{tikzpicture}
\end{center}
\caption{
The physicalistic model.
The hidden state $s$ contains information about the agent that is unknown to it.
The distribution $\mu$ is the agent's (subjective) \emph{environment model},
and $\pi$ its (deterministic) policy.
The agent models itself through the
beliefs about (future) actions given by its environment model $\mu$.
Interaction with the environment at time step $t$ occurs through
an action $a_t$ chosen by the agent
and a percept $e_t$ returned by the environment.
}
\label{fig:physicalistic-agent-model}
\end{figure}
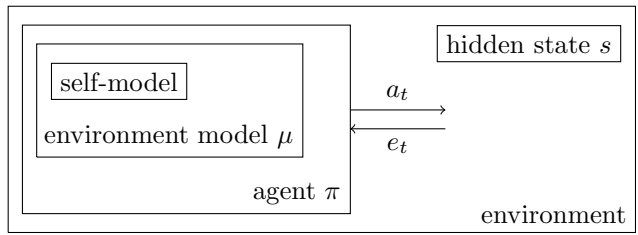

%\paragraph{The Physicalistic Model}
In the physicalistic model
the agent is embedded inside the \emph{environment},
as depicted in \autoref{fig:physicalistic-agent-model}.
The environment has a \emph{hidden state} that
contains information about the agent
that is inaccessible to the agent itself.
The agent has an \emph{environment model}
that describes the behavior of the environment given the hidden state
and includes beliefs about the agent's own future actions (thus modeling itself).

%\paragraph{Actions as Means and Evidence.}
Physicalistic agents may view their actions in two ways:
as their selected output,
and as consequences of properties of the environment.
This leads to significantly more complex problems of inference and decision making,
with actions simultaneously being both means to influence the environment
and evidence about it.
For example, looking at cat pictures online
may simultaneously be a \emph{means} of procrastination, and
\emph{evidence} of bad air quality in the room.

Dualistic decision making in a known environment is straightforward
calculation of expected utilities.
This is known as Savage decision theory~\cite{Savage:1972}.
For non-dualistic decision making
two main approaches are offered by the decision theory literature:
\emph{causal decision theory}~(CDT)%
~\cite{GH:1978,Lewis:1981,Skyrms:1982,Joyce:1999,SEP:CDT} and
\emph{evidential decision theory}~(EDT)~\cite{Jeffrey:1983, SEP:EDT,Ahmed:2014}.
EDT and CDT both take actions that maximize expected utility,
but differ in the way this expectation is computed:
EDT uses the action under consideration as evidence about the environment
while CDT does not.
\autoref{sec:dt} formally introduces these decision algorithms.

Our contribution is to formalize and explore a decision-theoretic
setting with a physicalistic reinforcement learning agent
interacting \emph{sequentially} with an environment that it is embedded in
(\autoref{sec:informed-decisions}).
Previous work on non-dualistic decision theories has focused on
\emph{one-shot} situations.
We find that there are two natural extensions of EDT to the sequential
case, depending on whether the agent updates beliefs based on its next
action or its entire policy. CDT has only one natural extension.
We extend two famous \emph{Newcomblike problems} to the sequential setting
to illustrate the differences between our
(generalized) decision theories.

% Outline of paper
\autoref{sec:discussion} summarizes our results and outlines future
directions.
A list of notation can be found on \hyperref[app:notation]{page~\pageref*{app:notation}}
\iftechreport
and \appendixref{app:examples} contains formal details to our examples.
\else
and the formal details of the examples can be found in
the technical report~\cite{ELH:2015tech}.
\fi

%%%%%%%%%%%%%%%%%%%%%%%%%%%%%%%%%%%%%%%%%%%%%%%%%%%%%%%%%%%%%%%
\section{One-Shot Decision Making}
\label{sec:dt}

%\paragraph{Setup.}
In a \emph{one-shot decision problem},
we take one \emph{action $a \in \A$}, receive a \emph{percept $e \in \E$}
(typically called \emph{outcome} in the decision theory literature)
and get a \emph{payoff $u(e)$} according to the 
\emph{utility function $u: \E \to [0, 1]$}.
We assume that the set of actions $\A$ and the set of percepts $\E$ are finite.
Additionally,
the environment contains a \emph{hidden state} $s \in \S$.
The hidden state holds information that is inaccessible to the agent
at the time of the decision,
but may influence the decision and the percept.
Formally, the environment is given by
a probability distribution $P$ over the hidden state, the action, and the percept
that factors according to a causal graph~\cite{Pearl:2009}.

%A fundamental concept is the concept of \emph{causal graphs}.
A \emph{causal graph} over the random variables $x_1,\dots,x_n$ is
a directed acyclic graph with nodes $x_1,\dots,x_n$.
To each node $x_i$ belongs a
probability distribution $P(x_i\mid \pa_i)$,
where $\pa_i$ is the set of parents of $x_i$ in the graph.
It is natural to identify the causal graph with 
the factored distribution $P(x_1,\dots, x_n)=\prod_{i=1}^nP(x_i\mid \pa_i)$.
Given such a causal graph/factored distribution, 
we define the \emph{$\doo$-operator} as
\begin{equation}\label{eq:def-do}
  P(x_1, \ldots, x_{j-1}, x_{j+1},\ldots, x_n \mid \doo(x_j := b))
= \prod_{\substack{i=1\\i\not=j}}^n P(x_i \mid \pa_i)
\end{equation}
where $x_j$ is set to $b$ wherever it occurs in $\pa_i$, $1\leq i\leq n$.
The result is a new probability distribution
that can be marginalized and conditioned in the standard way.
Intuitively, intervening on node $x_j$ means
ignoring all incoming arrows to $x_j$, as the effects they represent
are no longer relevant when we intervene;
the factor $P(x_j\mid \pa_j)$ representing the ingoing influences
to $x_j$ is therefore removed in the right-hand side of \eqref{eq:def-do}.
Note that the $\doo$-operator is only defined for distributions
for which a causal graph has been specified.
See \cite[Ch.\ 3.4]{Pearl:2009} for details.

%%%%%%%%%%%%%%%%%%%%%%%%%%%%%%%%%%%%%%%%%%%%%%%%%%%%%%%%%%%%%%%
\subsection{Savage Decision Theory}
\label{sec:sdt}

%\paragraph{SDT.}
In the \emph{dualistic} formulation of decision theory,
we have a function $P$ that takes an action $a$ and returns
a probability distribution $P_a$ over percepts.
\emph{Savage decision theory}~(SDT)~\cite{Savage:1972,SEP:EDT}
takes actions according to
\begin{equation}\label{eq:SDT}
\argmax_{a \in \A} \sum_{e \in \E} P_a(e) u(e).
\tag{SDT}
\end{equation}

%-------------------------------%
\begin{figure}[t]
%-------------------------------%
\begin{center}
\begin{tikzpicture}[node distance=20mm, auto]
\node[c] (a1) {$a$};
\node[c,right of=a1] (e1) {$e$};
\node[c,above of=a1] (hidden) {$s$};
\draw[->] (hidden) to (a1);
\draw[->] (hidden) to (e1);
\draw[->] (a1) to (e1);
\end{tikzpicture}
\end{center}
\caption{
The causal graph $P(s,a,e)=P(s)P(a\mid s)P(e\mid s,a)$ for one-step decision making.
The hidden state $s$ influences both the decision maker's action $a$
and the received percept $e$.
}\label{fig:one-step-model}
\end{figure}
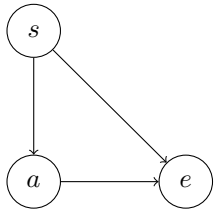

%\paragraph{Physicalism.}
In the dualistic model it is usually conceptually clear what $P_a$ should be.
In the physicalistic model
the environment model takes the form of a causal graph
over a hidden state $s$, action $a$, and percept $e$,
as illustrated in \autoref{fig:one-step-model}.
According to this causal graph,
the probability distribution $P$ factors causally into
$P(s, a, e) = P(s) P(a \mid s) P(e \mid s, a)$.
The hidden state is not independent of the decision maker’s action and
Savage's model is not directly applicable
since we do not have a specification of $P_a$.
% See Eells, Ellery. 1982. Rational Decision and Causality. page 74
How should decisions be made in this context?
The literature focuses on two answers to this question:
CDT and EDT.

%%%%%%%%%%%%%%%%%%%%%%%%%%%%%%%%%%%%%%%%%%%%%%%%%%%%%%%%%%%%%%%
\subsection{Causal and Evidential Decision Theory}
\label{sec:cdt-and-edt}

%\paragraph{Literature on Decision Theory.}
The literature on causal and evidential decision theory is vast, and
we give only a very superficial overview
that is intended to bring the reader up to speed on the basics.
See \cite{SEP:EDT,SEP:CDT} and references therein
for more detailed introductions.

%\paragraph{EDT.}
\emph{Evidential decision theory} (endorsed in \cite{Jeffrey:1983, Ahmed:2014})
considers the probability of the percept $e$ \emph{conditional on}
taking the action $a$:
\begin{equation}\label{eq:EDT}
\argmax_{a \in \A} \sum_{e \in \E} P(e \mid a)\, u(e)
\quad\text{with}\quad
P(e \mid a) = \sum_{s \in \S} P(e \mid s, a) P(s \mid a)
\tag{EDT}
\end{equation}

\emph{Causal decision theory} has several formulations~\cite{GH:1978,Lewis:1981,Skyrms:1982,Joyce:1999};
we use the one given in \cite{Skyrms:1982},
with Pearl's calculus of causality~\cite{Pearl:2009}.
According to CDT, the probability of a percept $e$ is given by
the \emph{causal intervention} of performing action $a$
on the causal graph from \autoref{fig:one-step-model}:
\begin{equation}\label{eq:CDT}
\argmax_{a \in \A} \sum_{e \in \E} P(e \mid \doo(a))\, u(e)
\quad\text{with}\quad
P(e \mid \doo(a)) = \sum_{s \in \S} P(e \mid s, a) P(s)
\tag{CDT}
\end{equation}
where  $P(e\mid \doo(a))$ follows from \eqref{eq:def-do} and marginalization over $s$.

The difference between CDT and EDT is how the action affects
the belief about the hidden state.
EDT assigns credence $P(s\mid a)$ to the hidden state $s$ if action $a$
is taken, while CDT assigns credence $P(s)$.
A common argument for CDT is that an action under my direct control should
not influence my belief about things that are not causally affected by the action.
Hence $P(s)$ should be my belief in $s$, and not $P(s\mid a)$.
(By assumption, the action does not \emph{causally} affect the hidden state.)
EDT might reply that if action $a$ does not have the same likelihood under
all hidden states $s$, then action $a$ should indeed inform me about 
the hidden state, regardless of causal connection.
The following two classical examples from the decision theory literature
describe situations where CDT and EDT disagree.
\iftechreport
A formal definition of these examples can be found in
\appendixref{app:examples}.
\else
A formal definition of these examples can be found in
the technical report~\cite{ELH:2015tech}.
\fi

%-------------------------------%
\begin{example}[{Newcomb's Problem~\cite{Nozick:1969}}]
\label{ex:Newcomb}
%-------------------------------%
In Newcomb's Problem there are two boxes:
an opaque box that is either empty or contains one million dollars and
a transparent box that contains one thousand dollars.
The agent can choose between taking only the opaque box (`one-boxing')
and taking both boxes (`two-boxing').
The content of the opaque box is determined by a prediction
about the agent's action by a very reliable predictor:
if the agent is predicted to one-box, the box contains the million, and
if the agent is predicted to two-box, the box is empty.
In Newcomb's problem
EDT prescribes one-boxing because one-boxing is evidence that
the box contains a million dollars.
In contrast,
CDT prescribes two-boxing because two-boxing dominates one-boxing:
in either case we are a thousand dollars richer,
and our decision cannot causally affect the prediction.
Newcomb's problem has been raised as a critique to CDT,
but many philosophers insist that two-boxing is in fact
the rational choice,%
\footnote{%
In a 2009 survey,
31.4\% of philosophers favored two-boxing,
and 21.3\% favored one-boxing (931 responses);
%61\% of decision theorists favored two-boxing,
%and 26\% favored one-boxing (31 responses).
see \url{http://philpapers.org/surveys/results.pl}.
Is that the reason there are so few wealthy philosophers?}
even if it means you end up poor.

Note how the decision depends on whether the action influences the belief
about the hidden state (the contents of the opaque box) or not.
\end{example}

%\paragraph{Newcomb is relevant}
Newcomb's problem may appear as
an unrealistic thought experiment.
However, we argue that problems with similar structure are fairly common.
The main structural requirement is that $P(s \mid a) \neq P(s)$
for some state or event $s$ that is not causally affected by $a$.
In Newcomb's problem the predictor's ability to guess the action
induces an `information link' between actions and hidden states.
If the stakes are high enough, the predictor does not have to be
much better than random in order to generate a
\emph{Newcomblike decision problem}.
Consider for example
spouses predicting the faithfulness of their partners,
employers predicting the trustworthiness of their employees, or
parents predicting their children's intentions.
For AIs, the potential for accurate predictions is even greater,
as the predictor may have access to the AI's source code.
Although rarely perfect,
all of these predictions are often substantially better than random.

To counteract the impression that EDT is generally superior to CDT,
we also discuss the \emph{toxoplasmosis problem}.

%-------------------------------%
\begin{example}[{Toxoplasmosis Problem~\cite{Altair:2013}}]\!\!\!\footnote{%
Historically, this problem has been known as
the \emph{smoking lesion problem}~\cite{Egan:2007}.
We consider the smoking lesion formulation confusing,
because today it is universally known that
smoking \emph{does} cause lung cancer.%
}
\label{ex:toxoplasmosis}
%-------------------------------%
This problem takes place in a world in which there is a certain parasite that
causes its hosts to be attracted to cats,
in addition to uncomfortable side effects.
The agent is handed an adorable little kitten and
is faced with the decision of whether or not to pet it.
Petting the kitten feels nice and
therefore yields more utility than not petting it.
However, people suffering from the parasite are more likely to pet the kitten.
Petting the kitten is evidence of having the parasite, so EDT recommends against it.
CDT correctly observes that petting the kitten does not \emph{cause}
the parasite,
and is therefore in favor of petting.
\end{example}

%\paragraph{CDT and EDT Recast as SDT.}
\hyperref[ex:Newcomb]{Newcomb's problem} and
the \hyperref[ex:toxoplasmosis]{toxoplasmosis problem}
cannot be properly formalized in SDT,
because SDT requires the percept-probabilities $P_a$ to be specified,
but it is not clear what the right choice of $P_a$ would be.
However, both CDT and EDT can be recast in the context of \ref{eq:SDT}
by setting $P_a$ to be
$P(\,\cdot \mid \doo(a))$ and $P(\,\cdot \mid a)$ respectively.
Thus we could say that
the formulation given by Savage needs a specification of the environment
that tells us whether to act evidentially, causally, or otherwise.

%%%%%%%%%%%%%%%%%%%%%%%%%%%%%%%%%%%%%%%%%%%%%%%%%%%%%%%%%%%%%%%
\section{Sequential Decision Making}
\label{sec:informed-decisions}

In this section we extend CDT and EDT to the sequential case.
We start by formally specifying the physicalistic model depicted 
in \autoref{fig:physicalistic-agent-model} in the first subsection,
and discuss problems with time consistency in \autoref{ssec:time-consistency},
before defining the extensions proper in \autoref{ssec:sedt} and
\ref{ssec:scdt}.
The \hyperref[sec:hidden-states]{final subsection} dissects the role
of the hidden state.

%%%%%%%%%%%%%%%%%%%%%%%%%%%%%%%%%%%%%%%%%%%%%%%%%%%%%%%%%%%%%%%
\subsection{The Physicalistic Model}
\label{sec:physicalistic-model}

%\paragraph{Agent-Environment Interaction}
For the remainder of this paper, we assume that the agent interacts
sequentially with an environment.
At time step $t$ the agent chooses an \emph{action} $a_t \in \A$ and
receives a \emph{percept} $e_t \in \E$
which yields a \emph{utility} of $u(e_t) \in \mathbb{R}$;
the cycle then repeats for $t + 1$.
A \emph{history} is an element of $\H$.
We use $\ae \in \A \times \E$ to denote one interaction cycle,
and $\ae_{<t}$ to denote a history of length $t - 1$.
The percepts between time $t$ and time $m$ are denoted $e_{t:m}$.
A \emph{policy} is a function that maps a history $\ae_{<t}$ to
the next action $a_t$.
We only consider deterministic policies.

%-------------------------------%
\begin{figure}[t]
%-------------------------------%
\begin{center}
\begin{tikzpicture}[node distance=20mm, auto]
\node[c] (a1) {$a_1$};
\node[c,right of=a1] (e1) {$e_1$};
\node[c,right of=e1] (a2) {$a_2$};
\node[c,right of=a2] (e2) {$e_2$};
\node[right of=e2,minimum height=2em] (dots) {\ldots};
\node[c,above of=a1] (hidden) {$s$};
\draw[->] (hidden) to (a1);
\draw[->] (hidden) to (e1);
\draw[->] (hidden) to (a2);
\draw[->] (hidden) to (e2);
\draw[->] (hidden) to (dots);
\draw[->] (a1) to (e1);
\draw[->,bend right] (a1) to (a2);
\draw[->,bend right] (a1) to (e2);
\draw[->,bend right] (a1) to (dots);
\draw[->] (e1) to (a2);
\draw[->,bend left] (e1) to (e2);
\draw[->,bend left] (e1) to (dots);
\draw[->] (a2) to (e2);
\draw[->,bend right] (a2) to (dots);
\draw[->] (e2) to (dots);
\end{tikzpicture}
\end{center}
\vspace{-7mm}
\caption{
The (infinite) causal graph for a sequential environment.
Each action $a_t$ and each percept $e_t$ is represented by a node in the causal graph.
Actions and percepts affect all subsequent actions and percepts:
causality follows time.
The hidden state $s$ is only ever indirectly (partially) observed.
}\label{fig:model-with-hidden-state}
\end{figure}
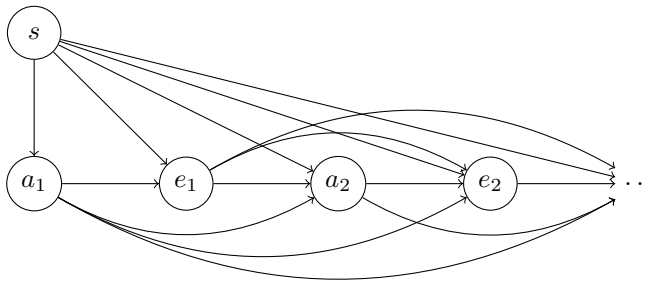

%\paragraph{The Environment Model.}
We assume that the agent is given an environment model $\mu$,
but knows neither the hidden state $s$ nor
its own future actions.
The unknown hidden state may influence both percepts and actions.
Actions and percepts in turn influence the entire future.
The environment model $\mu$ is given by
a probability distribution over hidden states and histories
that factors as
\begin{equation}\label{eq:mu-factorized}
  \mu(s, \ae_{<t})
= \mu(s) \prod_{i=1}^{t-1} \mu(a_i \mid s, \ae_{<i}) \mu(e_i \mid s, \ae_{<i}a_i)
\end{equation}
for any $t \in \mathbb N$.
While such a factorization is possible for any distribution,
we additionally demand that this factorization is \emph{causal}
according to the causal graph in \autoref{fig:model-with-hidden-state}.
The distribution $\mu(a_t\mid s,\ae_{<t})$ gives the likelihood of the agent's 
own actions provided a hidden state $s\in\S$
(for example, the prior probability of an infected agent petting the
kitten in the toxoplasmosis problem above). 
For technical reasons, this distribution must always leave some uncertainty about the actions:
if the environment model assigned probability zero for an action $a'$,
the agent could not deliberate taking action $a'$
since $a'$ could not be conditioned on.
Formally, we require $\mu(\,\cdot\mid s)$ to be \emph{action-positive} for all $s\in\S$:
\begin{equation}\label{eq:action-positive}
\forall \ae_{<t}a_t \in \H \times \A.\;
  \big( \mu(\ae_{<t}\mid s) > 0 \implies \mu(a_t \mid s, \ae_{<t}) > 0 \big)
\end{equation}

%\paragraph{Environment vs.\ Environment Model}
The distribution $\mu$ is a \emph{model} of the environment,
a belief held by the agent,
but not the distribution from which the actual history is drawn.
The actual history is distributed according to the true environment
distribution.
Because the environment contains the agent,
the agent's algorithm might get modified by it and
the actions that the agent actually ends up taking
might not be the actions that were planned.
In the end, model and reality will disagree:
for example, we simultaneously assume the agent's
policy $\pi$ to be deterministic
and the environment model to be action positive. % \eqref{eq:action-positive}.
Nevertheless, we assume
the given environment model is \emph{accurate} in the sense that
it faithfully represents the environment in the ways relevant to the agent.
In other words, we are interested in problems that arise during planning,
not problems that arise due to poor modeling.

%%%%%%%%%%%%%%%%%%%%%%%%%%%%%%%%%%%%%%%%%%%%%%%%%%%%%%%%%%%%%%%
\subsection{Time Consistency}
\label{ssec:time-consistency}

%\paragraph{Lifetime.}
When planning for the infinite future
we need to make sure that utilities do not sum to infinity;
typically this is achieved with discounting.
Here, we simplify by fixing a finite $m \in \mathbb{N}$
to be the agent's \emph{lifetime}:
the agent cares about the sum of the utilities of all percepts
$e_1 \ldots e_m$ until and including time step $m$,
but does not care what happens after that
(presumably the agent is then retired).

In sequential decision theory
we need to plan the next $m-t$ actions in time step $t$.
We plan what we would do for all possible future percepts $e_{t:m}$
by choosing a policy $\pi: \H \to \A$ that specifies which action we take
depending on how the history plays out.
For example, we take action $a_t$,
and when we subsequently receive the percept $e_t$,
we plan to take action $a_{t+1}$.
Problems arise once we get to the next step and
even tough we \emph{did} take action $a_t$ and
the percept \emph{did} turn out to be $e_t$,
we change our mind and take a different action $\hat{a}_{t+1}$.
This is called \emph{time inconsistency}.
Time inconsistency is an artifact of bad planning
since the agent incorrectly anticipates her own actions.
The choice of discounting can lead to time inconsistency:
a sliding fixed-size horizon is time inconsistent,
but a fixed finite lifetime is time consistent~\cite{LH:2014discounting}.

We achieve time consistency by using a fixed finite lifetime, and
by calculating decisions recursively using value functions.
A \emph{value function $V_{\mu,m}^\pi$} is
a function of type $(\H \cup (\H \times \A)) \to \mathbb{R}$.
It gives an estimate of future reward: $V_{\mu,m}^\pi(\ae_{<t})$ and
$V_{\mu,m}^\pi(\ae_{<t}a_t)$ are estimates of
how much reward the policy $\pi$ will
obtain in environment $\mu$ within
lifetime $m$ subsequent to history $\ae_{<t}$ and $\ae_{<t}a_t$ respectively.
For any history $\ae_{<t}$,
we define $V_{\mu,m}^\pi(\ae_{<t}) := V_{\mu,m}^\pi(\ae_{<t}\pi(\ae_{<t}))$.
We say that a policy $\pi$ is
\emph{optimal and time consistent for the value function $V_{\mu,m}$} iff
$\pi(\ae_{<t})=\arg\max_a V_{\mu,m}^\pi(\ae_{<t} a)$
for all histories $\ae_{<t} \in (\A \times \E)^{t-1}$ and all $t \leq m$.

%%%%%%%%%%%%%%%%%%%%%%%%%%%%%%%%%%%%%%%%%%%%%%%%%%%%%%%%%%%%%%%
\subsection{Sequential Evidential Decision Theory}
\label{ssec:sedt}

Evidential decision theory assigns probability $P(e \mid a)$
to action $a$ resulting in percept $e$ (\autoref{sec:cdt-and-edt}). 
There are two ways to generalize this to the sequential setting,
depending on whether we use only the next action or the whole future policy
as evidence for the next percept.

%-------------------------------%
\begin{definition}[Action-Evidential Decision Theory]
\label{def:action-evidential-value}
%-------------------------------%
The \emph{action-evidential value of a policy $\pi$
with lifetime $m$ in environment $\mu$
given history $\ae_{<t}a_t$} is
\begin{equation}
   V^{\aevi,\pi}_{\mu,m}(\ae_{<t}a_t)
:= \sum_{e_t}\mu(e_t \mid \ae_{<t}a_t)
     \Big( u(e_t) + V^{\aevi,\pi}_{\mu,m}(\ae_{<t}a_te_t) \Big)
\tag{SAEDT}\label{eq:SAEDT}
\end{equation}
and $V^{\aevi,\pi}_{\mu,m}(\ae_{<t}a_t) := 0$ for $t > m$.
\emph{Sequential Action-Evidential Decision Theory (SAEDT)} prescribes
adopting an optimal and time consistent policy $\pi$ for $V^{\aevi}_{\mu,m}$.
\end{definition}

It may be argued that \ref{eq:SAEDT}
does not take all available (deliberative) information into account.
When considering the consequences of an action, future developments
of the environment-policy interactions could also be used as evidence.
That is, we could condition not only on the next action,
but on the future policy as a whole (within the lifetime).
In order to define conditional probabilities with respect to (deterministic) 
policies, we define the following events. 
For a given policy $\pi$,
let $\Pi_{t:m}$ be the set of all strings consistent with $\pi$
between time step $t$ and $m$:
\[
\Pi_{t:m} := \{ \ae_{1:\infty} \mid \forall t \leq i \leq m.\; \pi(\ae_{<i}) = a_i \}
\]
The likelihood of a next percept $e_t$ provided a history $\ae_{<t}$
and a (future) policy $\pi$ followed from time step $t$ until lifetime $m$ 
(denoted $\pi_{t:m}$) is then defined as
\begin{equation}
\mu(e_t\mid \ae_{<t},\pi_{t:m}) := \mu(e_t\mid \ae_{<t}\cap \Pi_{t:m}).
\label{eq:policy-cond}
\end{equation}
This is an \emph{atemporal} conditional because we are conditioning on
future actions up until the end of the agent's lifetime.
The conditional \eqref{eq:policy-cond} is well-defined
because we only take the actions from time step $t$ to $m$ into account;
conditioning on policies with infinite lifetime
leads to technical problems
because such policies typically have $\mu$-measure zero.

%-------------------------------%
\begin{definition}[Policy-Evidential Decision Theory]
\label{def:policy-evidential-value}
%-------------------------------%
The \emph{policy-evidential value of a policy $\pi$
with lifetime $m$ in environment $\mu$
given history $\ae_{<t}a_t$} is
\begin{equation}
   V^{\pevi,\pi}_{\mu,m}(\ae_{<t}a_t)
:= \sum_{e_t}\mu(e_t \mid \ae_{<t}a_t, \pi_{t+1:m})\cdot
     \Big( u(e_t) + V^{\pevi,\pi}_{\mu,m}(\ae_{<t}a_te_t) \Big)
\tag{SPEDT}\label{eq:SPEDT}
\end{equation}
and $V^{\pevi,\pi}_{\mu,m}(\ae_{<t}) := 0$ for $t > m$.
\emph{Sequential Policy-Evidential Decision Theory (SPEDT)} prescribes
adopting an optimal and time consistent policy $\pi$ for $V^{\pevi}_{\mu,m}$.
\end{definition}
For one-step decisions ($m = t+1$), SAEDT and SPEDT coincide.

To all our embedded agents, past actions constitute evidence about 
the hidden state.
For evidential agents, this principle is extended to future actions.
SAEDT and SPEDT differ in how far they extend it.
The action-evidential agent only updates his belief on
the action about to take place.
In that sense, he only updates his belief about the next
percept on events taking place \emph{before} this percept.
The policy-evidential agent takes the principle much further,
using ``thought-experiments'' of what action he \emph{would take
in hypothetical situations}, most of which will never be realized.
This is illustrated in the next example.

%-------------------------------%
\begin{example}[Sequential Toxoplasmosis]
\label{ex:sequential-toxoplasmosis}
%-------------------------------%
In our sequential variation of
the \hyperref[ex:toxoplasmosis]{toxoplasmosis problem}
the agent has some probability of encountering a kitten.
Additionally, the agent has the option of seeing a doctor (for a fee) and
getting tested for the parasite, which can then be safely removed.
In the very beginning, an SPEDT agent updates his belief on the fact that
if he encountered a kitten, he would not pet it,
which lowers the probability that he has the parasite
and makes seeing the doctor unattractive.
An SAEDT agent only updates his belief about the parasite
when he actually encounters a kitten,
and thus prefers seeing the doctor. See \autoref{fig:seq-toxo} for
more details and a graphical illustration.
\end{example}

\begin{figure}[t]
\centering
\begin{tikzpicture}[
	grow = right,  % alignment of characters
	level 1/.style = {sibling distance=35mm,level distance=10mm},
	level 2/.style = {sibling distance=15mm,level distance=21mm},
	level 3/.style = {sibling distance=10mm},
	level 5/.style = {level distance=10mm},
]
\node (init) {}
child {
	node (I) {Healthy}
	child {
		node (IE) {No doc}
		child {
			node (IEK) {Kitten (0)}
                        child {
				node (IEKQ) {Not pet}
				child { node[right] {Healthy, not pet (0)} }
			}
			child {
				node (IEKP) {Pet}
				child { node[right] {Healthy, pet (1)} }
			}
		}
	} child {
		node (ID) {Doc}
		child {
		 	node (ID0) {Healthy ($-4$)}
		}
	}
}
child {
	node (J) {Toxo}
	child {
		node (JE) {No doc}
		      child {
			  node (JEK) {Kitten (0)}
				child {
                                    node(JEKQ) {Not pet}
                                    child { node[right] {Sick, not pet ($-10$)} }
                                 }
                                child {
                                    node(JEKP) {Pet}
                                    child { node[right] {Sick, pet ($-9$)} }
                                    }
			} child {
				node (JES) {Sick ($-10$)}
			}
	} child {
		node (JD) {Doc}
		child {
			node (JD0) {Cured ($-4$)}
		}
	}
};
\path (init) to node[left] {0.5} (J);
\path (init) to node[left] {0.5} (I);
\path (J) to node[above] {0.5} (JD);
\path (J) to node[below] {0.5} (JE);
\path (I) to node[above] {0.5} (ID);
\path (I) to node[below] {0.5} (IE);
\path (JE) to node[below] {$0.2$} (JEK);
\path (JE) to node[above] {$0.8$} (JES);
\path (JEK) to node[above] {$0.8$}      (JEKP);
\path (JEK) to node[below] {$0.2$}      (JEKQ);
\path (IEK) to node[above] {$0.2$}      (IEKP);
\path (IEK) to node[below] {$0.8$}      (IEKQ);
\draw[dashed](I)to(J);
\draw[dashed, bend right](IE)to(JE);
\draw[dashed, bend left](IEK)to(JEK);
\draw[dashed, bend left](IEKP)to(JEKP);
\draw[dashed, bend right](IEKQ)to(JEKQ);
\end{tikzpicture}
\caption{%
One formalization of the sequential toxoplasmosis problem.
Dashed lines connect states indistinguishable to the agent.
The numbers on the edges indicate probabilities of the environment
model $\mu$,
and the numbers in parenthesis indicate utilities of the
associated percepts.
In the first step, the environment selects the hidden state
that is unknown to the agent.
The agent then decides whether to go to the doctor.
If he does not go, he
may encounter a kitten which he can choose to pet or not.
SAEDT and SPEDT will disagree whether going
to the doctor is the best option in this scenario.
\iftechreport
\appendixref{app:examples} 
\else
\cite{ELH:2015tech}
\fi
contains the full calculations.
}
\label{fig:seq-toxo}
\end{figure}
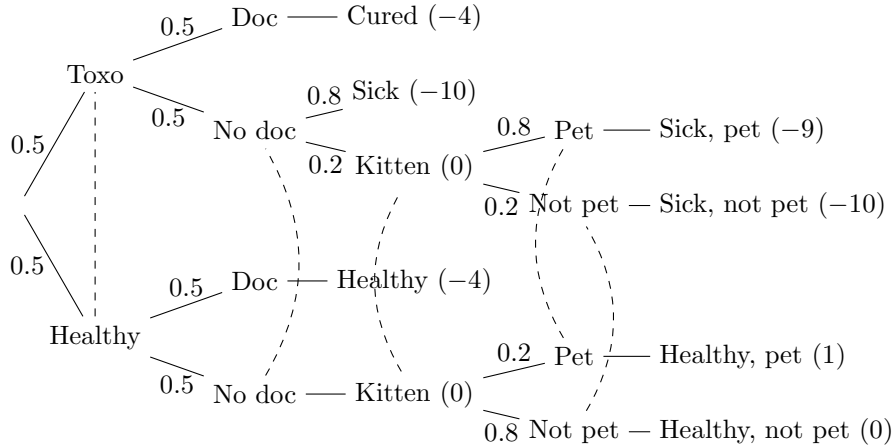

The observant reader may ask whether SPEDT could be enticed to
make some percepts unlikely by choosing improbable actions subsequent
to them. For example, could an SPEDT agent decide on a policy
of selecting highly improbable actions in case it rained to make
histories with rain less likely?
The answer is no, as most such
policies would not be time consistent.
If it does rain, the highly
improbable action would usually not the best one, and so the policy
would not be prescribed by \autoref{def:policy-evidential-value}.

%%%%%%%%%%%%%%%%%%%%%%%%%%%%%%%%%%%%%%%%%%%%%%%%%%%%%%%%%%%%%%%
\subsection{Sequential Causal Decision Theory}
\label{ssec:scdt}

In sequential causal decision theory
we ask what would happen if we causally intervened on the
node $a_t$ of the next action and fix it to $\pi(\ae_{<t})$
according to the policy $\pi$.
This is expressed by the notation $\doo(a_t := \pi(\ae_{<t}))$, 
or $\doo(\pi(\ae_{<t}))$ for short.

%-------------------------------%
\begin{definition}[Sequential Causal Decision Theory]
\label{def:causal-value}
%-------------------------------%
The \emph{causal value of a policy $\pi$ with lifetime $m$ in environment $\mu$
given history $\ae_{<t}a_t$} is
\begin{equation}\label{eq:SCDT}
   V^{\cau,\pi}_{\mu,m}(\ae_{<t}a_t)
:= \sum_{e_t \in \E} \mu(e_t \mid \ae_{<t}, \doo(a_t))
     \Big( u(e_t) + V^{\cau,\pi}_{\mu,m}(\ae_{<t}a_te_t) \Big)
\tag{SCDT}
\end{equation}
and $V^{\cau,\pi}_{\mu,m}(\ae_{<t}a_t) := 0$ for $t > m$.
\emph{Sequential Causal Decision Theory (SCDT)} prescribes
adopting an optimal and time consistent policy $\pi$ for $V^{\cau}_{\mu,m}$.
\end{definition}

For sequential evidential decision theory we discussed two versions 
\eqref{eq:SAEDT} and \eqref{eq:SPEDT}, based on next action and future
policy respectively.
In \ref{eq:SCDT} we perform
the causal intervention $\doo(a_t := \pi(\ae_{<t}))$.
We could also consider a policy-causal decision theory by replacing
$\mu(e_t \mid \ae_{<t}, \doo(a_t))$ with
$\mu(e_t \mid \ae_{<t}, \doo(\pi_{t:m}))$ in \autoref{def:causal-value}.
The causal intervention $\doo(\pi_{t:m}))$
of a policy $\pi$ between time step $t$ and time step $m$
is defined as as
\begin{equation}\label{eq:policy-causal}
   \mu(e_t \mid \ae_{<t}, \doo(\pi_{t:m}))
:= \!\!\!\sum_{e_{t+1:m}}\!\!\!
     \mu(e_{t:m} \mid \ae_{<t}, \doo(a_t := \pi(\ae_{<t}), \ldots, a_m := \pi(\ae_{<m}))).
\end{equation}
However, since the interventions are causal,
we do not get any extra evidence from the future interventions.
Therefore policy-causal decision theory is the same as
action-causal decision theory:

%-------------------------------%
\begin{proposition}[Policy-Causal = Action-Causal]
\label{prop:policy-cdt}
%-------------------------------%
For all histories $\ae_{<t} \in \H$ and all $e_t \in \E$,
we have
$
  \mu(e_t \mid \ae_{<t}, \doo(\pi_{t:m}))
= \mu(e_t \mid \ae_{<t}, \doo(\pi(\ae_{<t}))).
$
\end{proposition}
We defer the proof to the end of this section.
The following two examples illustrate the difference between
SCDT and SAEDT/SPEDT in sequential settings.

%-------------------------------%
\begin{example}[Newcomb with Precommitment]
\label{ex:Newcomb-with-precommitment}
%-------------------------------%
In this variation to \hyperref[ex:Newcomb]{Newcomb's problem}
the agent first has the option to pay \$300,000 to sign a contract that
binds the agent to pay \$2000 in case of two-boxing.
An SAEDT or SPEDT agent knows that he will one-box anyways
and hence has no need for the contract.
An SCDT agent knows that she favors two-boxing,
but signs the contract only if this occurs before the prediction is made
(so it has a chance of causally affecting the prediction).
With the contract in place, one-boxing is the dominant action,
and thus the SCDT agent is predicted to one-box.
\end{example}

%-------------------------------%
\begin{example}[Newcomb with Looking]
\label{ex:Newcomb-with-looking}
%-------------------------------%
In this variation to \hyperref[ex:Newcomb]{Newcomb's problem}
the agent may look into the opaque box before making the decision
which box to take.
An SCDT agent is indifferent towards looking because
she will take both boxes anyways.
However, an SAEDT or SPEDT agent will avoid looking into the box,
because once the content is revealed he two-boxes.
\end{example}

%%%%%%%%%%%%%%%%%%%%%%%%%%%%%%%%%%%%%%%%%%%%%%%%%%%%%%%%%%%%%%%
\subsection{Expansion over the Hidden State}
\label{sec:hidden-states}

The difference between sequential versions of
EDT and CDT is how they update their prediction of a next percept $e_t$
(Definitions \ref{def:action-evidential-value},
\ref{def:policy-evidential-value} and \ref{def:causal-value}).
The following proposition expands the different beliefs in terms of
the hidden state.

\begin{proposition}\label{prop:conditional}
For all histories $\ae_{<t} a_t e_t \in \H$ the following holds for the
next-percept beliefs of SAEDT, SPEDT and SCDT respectively:
\begin{align}
   \mu(e_t \mid \ae_{<t} a_t)
&= \sum_{s \in \S} \mu(s \mid \ae_{<t} a_t) \mu(e_t \mid s, \ae_{<t} a_t)
\label{eq:mu-action-condition} \\
   \mu(e_t \mid \ae_{<t}, \pi_{t:m})
&= \sum_{s \in \S} \mu(s \mid \ae_{<t}, \pi_{t:m}) \mu(e_t \mid s, \ae_{<t}, \pi_{t:m})
\label{eq:mu-policy-condition} \\
   \mu(e_t \mid \ae_{<t}, \doo(a_t))
&= \sum_{s \in \S} \mu(s \mid \ae_{<t}) \mu(e_t \mid s, \ae_{<t} a_t)
\label{eq:mu-do}
\end{align}
\end{proposition}
\begin{proof}
For the action-evidential conditional we take the joint distribution with $s$,
and then split off $e_t$:
\begin{align*}
   \mu(e_t \mid \ae_{<t} a_t)
 = \frac{\sum_{s \in \S} \mu(s, \ae_{<t} a_t e_t)}{\mu(\ae_{<t} a_t)}
&= \frac{\sum_{s \in \S} \mu(s, \ae_{<t} a_t) \mu(e_t \mid s, \ae_{<t} a_t)}{\mu(\ae_{<t} a_t)} \\
&= \sum_{s \in \S} \mu(s \mid \ae_{<t} a_t) \mu(e_t \mid s, \ae_{<t} a_t)
\end{align*}

Similarly for the policy-evidential conditional:
\begingroup
\allowdisplaybreaks
\begin{align*}
   \mu(e_t \mid \ae_{<t}, \pi_{t:m})
 &= \frac{\sum_{s \in \S} \mu( s, \ae_{<t} \pi(\ae_{<t}) e_t, \pi_{t+1:m})}{\mu(\ae_{<t},\pi_{t:m})}\\
 &= \frac{\sum_{s \in \S} \mu(s, \ae_{<t} \pi(\ae_{<t}), \pi_{t+1:m}) \mu(e_t \mid s, \ae_{<t} \pi(\ae_{<t}), \pi_{t+1:m})}{\mu(\ae_{<t}, \pi_{t:m})} \\
 &= \frac{\sum_{s \in \S} \mu(s, \ae_{<t}, \pi_{t:m}) \mu(e_t \mid s, \ae_{<t} \pi(\ae_{<t}), \pi_{t+1:m})}{\mu(\ae_{<t}, \pi_{t:m})} \\
 &= \sum_{s \in \S} \mu(s \mid \ae_{<t}, \pi_{t:m}) \mu(e_t \mid s, \ae_{<t} \pi(\ae_{<t}), \pi_{t+1:m})\\
&= \sum_{s \in \S} \mu(s \mid \ae_{<t}, \pi_{t:m}) \mu(e_t \mid s, \ae_{<t}, \pi_{t:m})
\end{align*}
\endgroup

For the causal conditional we turn to
the rules of the $\doo$-operator~\cite[Thm.\ 3.4.1]{Pearl:2009}.
The first equality below holds by definition.
In the denominator of the second equality we can use Rule~3 (deletion of actions)
to remove $\doo(a_t)$ because
the $\doo$-operator removes all incoming edges to $a_t$ and
makes $a_t$ independent of the history $\ae_{<t}$.
In the numerator of the second equality we use the definition of $\doo$~\eqref{eq:def-do}:
\begingroup
\allowdisplaybreaks
\begin{align*}
   \mu(e_t \mid \ae_{<t}, \doo(a_t))
&= \frac{\mu(\ae_{<t}, e_t \mid \doo(a_t))}{\mu(\ae_{<t} \mid \doo(a_t))} \\
&= \frac{\sum_{s \in \S} \mu(s, \ae_{<t}) \mu(e_t \mid s, \ae_{<t} a_t)}{\mu(\ae_{<t})} \\
&= \sum_{s \in \S} \mu(s \mid \ae_{<t}) \mu(e_t \mid s, \ae_{<t} a_t)
\tag*{\qedhere}
\end{align*}
\endgroup
\end{proof}

\autoref{prop:conditional} shows that between SCDT and SAEDT, 
the difference in opinion about $e_t$ only depends on differences in their 
(acausal) \emph{posterior belief} $\mu(s \mid \ldots)$
about the hidden state.
SCDT and SAEDT thus become equivalent in scenarios
where there is only one hidden state $s^*$ with $\mu(s^*)=1$, as this
renders $\mu(s^* \mid \ae_{<t}) = \mu(s^* \mid \ae_{<t}a_t) = \mu(s^*) = 1$.
SPEDT, on the other hand, may disagree with the other two also 
after a hidden state has been fixed.

From a problem modeler's perspective, it is also instructive to consider 
the effect of moving uncertainty between the hidden state and environmental
stochasticity.
For two different environment models $\mu$ and $\mu'$,
the action and percept probabilities may be identical
(i.e., $\mu(a_t\mid \ae_{<t})=\mu'(a_t\mid \ae_{<t})$ and 
$\mu(e_t\mid \ae_{<t}a_t)=\mu'(e_t\mid \ae_{<t}a_t)$)
even though $\mu$ and $\mu'$ have non-isomorphic sets of hidden states
$\S$ and $\S '$. 
For example, given any $\mu$, an environment model $\mu'$ with a single hidden state $s_0$,
$\mu'(s_0)=1$,
may be constructed from $\mu$ by 
$\mu'(s_0,\ae_{<t}) := \sum_{s\in\S}\mu(s,\ae_{<t})$.
The transformation will not affect SAEDT and SPEDT, as the
definitions of their value functions
only depends on the `observable' action-
and percept-probabilities $\mu(a_t\mid\ae_{<t})$ and $\mu(e_t\mid\ae_{<t}a_t)$
which are preserved between $\mu$ and $\mu'$.
But the transformation will change SCDT's behavior in any $\mu$ where
SCDT disagrees with SAEDT, as SCDT and SAEDT are equivalent in $\mu'$ that
only has a single hidden state.
That SCDT depends on what uncertainty is captured by the hidden state
is unsurprising given that the hidden state has a special place in 
the causal structure of the problem.
Ultimately, the modeler must decide what
uncertainty to put in the hidden state, and what to attribute to 
environmental stochasticity. A general principle for how to do this is
still an open question~\cite{SF:2014DT}.

The value functions of SAEDT, SPEDT and SCDT can be rewritten in
the following \emph{iterative forms}, where the latter form uses 
\autoref{prop:conditional}.
Numbers above equality signs reference a justifying equation.
Let $a_i := \pi(\ae_{<i})$ for $i \geq t$:
\iftechreport
\begingroup
\allowdisplaybreaks
\begin{align}
   V^{\aevi,\pi}_{\mu,m}(\ae_{<t})
&= \sum_{k=t}^m\sum_{e_{t:k}} u(e_k) \prod_{i=t}^k \mu(e_i\mid\ae_{<i}a_i)
\label{eq:saedt-iterative} \\
&\stackrel{\eqref{eq:mu-action-condition}}{=}
   \sum_{k=t}^m\sum_{e_{t:k}} u(e_k) \prod_{i=t}^k
     \sum_{s \in \S} \mu(s \mid \ae_{<i} a_i) \mu(e_i \mid s, \ae_{<i}a_i)
\label{eq:saedt-iterative-explicit} \\
   V^{\pevi,\pi}_{\mu,m}(\ae_{<t})
&= \sum_{k=t}^m\sum_{e_{t:k}} u(e_k) \prod_{i=t}^k \mu(e_i\mid\ae_{<i}, \pi_{i:m})
\label{eq:spedt-iterative} \\
&\stackrel{\eqref{eq:mu-policy-condition}}{=} \sum_{k=t}^m\sum_{e_{t:k}} u(e_k) \prod_{i=t}^k
     \sum_{s \in \S} \mu(s \mid \ae_{<i} \pi_{i:m}) \mu(e_i \mid s, \ae_{<i},\pi_{i:m})
\label{eq:spedt-iterative-explicit}\\
   V^{\cau,\pi}_{\mu,m}(\ae_{<t})
&= \sum_{k=t}^m \sum_{e_{t:k}} u(e_k) \prod_{i=t}^k \mu(e_i \mid \ae_{<i}, \doo(a_i))
\label{eq:scdt-iterative} \\
&\stackrel{\eqref{eq:mu-do}}{=}
   \sum_{k=t}^m \sum_{e_{t:k}} u(e_i) \prod_{i=t}^k
     \sum_{s \in \S} \mu(s \mid \ae_{<i}) \mu(e_i \mid s, \ae_{<i} a_i)
\label{eq:scdt-iterative-explicit}
\end{align}
\endgroup
\else
\begingroup
\allowdisplaybreaks
\begin{align*}
   V^{\aevi,\pi}_{\mu,m}(\ae_{<t})
&= \sum_{k=t}^m\sum_{e_{t:k}} u(e_k) \prod_{i=t}^k \mu(e_i\mid\ae_{<i}a_i) \\
&\stackrel{\eqref{eq:mu-action-condition}}{=}
   \sum_{k=t}^m\sum_{e_{t:k}} u(e_k) \prod_{i=t}^k
     \sum_{s \in \S} \mu(s \mid \ae_{<i} a_i) \mu(e_i \mid s, \ae_{<i}a_i)\\
   V^{\pevi,\pi}_{\mu,m}(\ae_{<t})
&= \sum_{k=t}^m\sum_{e_{t:k}} u(e_k) \prod_{i=t}^k \mu(e_i\mid\ae_{<i}, \pi_{i:m})\\
&\stackrel{\eqref{eq:mu-policy-condition}}{=} \sum_{k=t}^m\sum_{e_{t:k}} u(e_k) \prod_{i=t}^k
     \sum_{s \in \S} \mu(s \mid \ae_{<i} \pi_{i:m}) \mu(e_i \mid s, \ae_{<i},\pi_{i:m})\\
   V^{\cau,\pi}_{\mu,m}(\ae_{<t})
&= \sum_{k=t}^m \sum_{e_{t:k}} u(e_k) \prod_{i=t}^k \mu(e_i \mid \ae_{<i}, \doo(a_i))\\
&\stackrel{\eqref{eq:mu-do}}{=}
   \sum_{k=t}^m \sum_{e_{t:k}} u(e_i) \prod_{i=t}^k
     \sum_{s \in \S} \mu(s \mid \ae_{<i}) \mu(e_i \mid s, \ae_{<i} a_i)
\end{align*}
\endgroup
\fi

\begin{proof}[Proof of \autoref{prop:policy-cdt}]
By the definition \eqref{eq:policy-causal} of $\doo(\pi_{t:m})$,
\begin{align*}
   \mu(e_t \mid \ae_{<t}, \doo(\pi_{t:m}))
& = \!\!\sum_{e_{t+1:m}}\!
     \mu(e_{t:m} \mid \ae_{<t}, \doo(a_t := \pi(\ae_{<t}), \ldots, a_m := \pi(\ae_{<m})))\\
& = \!\!\!\!\sum_{s, e_{t+1:m}}\!\!\!
     \mu(s\mid \ae_{<t})\mu(e_{t:m} \mid s, \ae_{<t}, \doo(\pi(\ae_{<t}),\ldots,\pi(\ae_{<m}))) \\
& \stackrel{\eqref{eq:def-do}}{=} 
\!\!\!\!\!\sum_{s,e_{t+1:m}} \!\!\mu(s\mid\ae_{<t}) \prod_{i=t}^m\mu(e_i\mid s, \ae_{<i}\pi(\ae_{<i}))\\ 
&=
    \sum_{s}\mu(s\mid\ae_{<t})\mu(e_t\mid s,\ae_{<t}\pi(\ae_{<t}))\\
&\stackrel{\eqref{eq:mu-do}}{=} \mu(e_t\mid \ae_{<t}, \doo(\pi(\ae_{<t})))
\end{align*}
The second equality follows from the equivalence
$P(\,\cdot\,) = \sum_s P(s)P(\,\cdot \mid s)$
applied to the distribution $\mu(\,\cdot \mid \ae_{<t}, \doo(a_t:=\pi(\ae_{<t}),\ldots,a_m:=\pi(\ae_{<m})))$,
and the third equality by (repeated) application of \eqref{eq:def-do}
to $\mu(\ae_{t:m}\mid s, \ae_{<t}) = \prod_{i=t}^m\mu(a_i\mid s,\ae_{<i})\mu(e_i\mid s, \ae_{<i}a_i)$.
\end{proof}

%%%%%%%%%%%%%%%%%%%%%%%%%%%%%%%%%%%%%%%%%%%%%%%%%%%%%%%%%%%%%%%
\section{Discussion}
\label{sec:discussion}

\begin{table}[t]
\begin{center}
\setlength{\tabcolsep}{1mm} % extra space between columns
\begin{tabular}{llll}
    & \ref{eq:SAEDT}          & \ref{eq:SPEDT}          & \ref{eq:SCDT} \\
\hyperref[ex:Newcomb]{Nwcb}
    & {\it 1-box}             & {\it 1-box}             & 2-box \\
\hyperref[ex:Newcomb-with-precommitment]{Nwcb w/ precommit}
    & {\it not commit, 1-box} & {\it not commit, 1-box} & commit, 1-box \\
\hyperref[ex:Newcomb-with-looking]{Nwcb w/ looking}
    & not look, 1-box         & not look, 1-box         & indifferent, 2-box \\
\hyperref[ex:toxoplasmosis]{Toxoplasmosis}
    & not pet                 & not pet                 & {\it pet} \\
\hyperref[ex:sequential-toxoplasmosis]{Seq.\ Toxoplasmosis}
    & doc, not pet            & no doc, not pet         & {\it doc, pet} \\
\end{tabular}
\end{center}
\caption{%
Decisions made by \ref{eq:SAEDT}, \ref{eq:SPEDT} and \ref{eq:SCDT}
in \autoref{ex:Newcomb}, \autoref{ex:toxoplasmosis},
\autoref{ex:sequential-toxoplasmosis},
\autoref{ex:Newcomb-with-precommitment}, and
\autoref{ex:Newcomb-with-looking}.
The latter three examples are sequential.
Winning moves are in italics;
in Newcomb with looking the winning move is to be indifferent and one-box.
Because Savage decision theory is dualistic,
these problems cannot be properly formalized in it.
}\label{tab:experiments}
\end{table}

Our paper is a first stab at the problem
of how physicalistic agents should make sequential decisions.
\ref{eq:CDT} and \ref{eq:EDT} provide an existing basis
for non-dualistic decision making,
which we extended to the sequential setting.
There are two natural ways for making sequential evidential decisions:
do I update my beliefs about the hidden state based on
my next action (`what I do next', \ref{eq:SAEDT}) or
my whole policy (`the kind of agent I am', \ref{eq:SPEDT})?
By \autoref{prop:policy-cdt},
this distinction does not exist for causal decision theory,
because with that theory
the agent does not consider its own actions evidence at all.
Therefore we have only one version of sequential causal decision theory,
\ref{eq:SCDT}.

%\paragraph{Implementation.}
To illustrate the differences between the decision theories,
we discussed three variants of Newcomb's problem
(\autoref{ex:Newcomb}, \autoref{ex:Newcomb-with-precommitment}, and
\autoref{ex:Newcomb-with-looking})
and two variants of the toxoplasmosis problem
(\autoref{ex:toxoplasmosis} and \autoref{ex:sequential-toxoplasmosis}).
\iftechreport
The formal specification of these examples
can be found in \appendixref{app:examples}.
\fi
We implemented \ref{eq:SCDT}, \ref{eq:SAEDT}, and \ref{eq:SPEDT};
\autoref{tab:experiments} shows their behavior on those examples.\footnote{%
Source code available at \url{http://jan.leike.name/}.
}

%\paragraph{Which decision theory is better?}
So which decision theory is better?
The answer to this question depends on
which decision you consider to be \emph{correct} (or even \emph{rational})
in each of the problems.
We posit that ultimately,
what counts is not whether your decision algorithm is theoretically pleasing,
but \emph{whether you win}.
Winning means getting the most utility.
If maximizing utility involves making crazy decisions,
then this is what you should do!

%\paragraph{Neither CDT nor EDT win.}
In Newcomb's problem, winning means one-boxing, because you end up richer.
In the toxoplasmosis problem, winning means petting the kitten,
because that yields more utility.
(S)CDT performs suboptimally in the Newcomb variations,
while the evidential decision theories perform suboptimally
in the toxoplasmosis variations.
This entails that neither CDT nor EDT are the final answer
to the problem of non-dualistic decision making.

%\paragraph{CDT and EDT are not fully physicalistic.}
Furthermore, neither CDT nor EDT agents are fully physicalistic:
they do not model the environment to contain themselves~\cite{SF:2014DT}.
For example, when playing a prisoner's dilemma
against your own source code~\cite{SF:2015UDT},
your opponent defects if and only if you defect.
This \emph{logical} connection between your action and your opponent's
is disregarded in the formalization based on
causal graphical models that we discuss here because it is not causal.

%\paragraph{Recent Physicalistic Decision Theories.}
\emph{Timeless decision theory}~\cite{Yudkowsky:2010tdt} and
\emph{updateless decision theory}~\cite{SF:2014DT} are recent
attempts of more physicalistic  decision theories.
However, so far both have eluded explicit formalization~\cite{SF:2015UDT}.
We conclude that
finding a physicalistic decision theory remains an important open problem
in artificial intelligence research.

\paragraph{Acknowledgements.}
This work was in part supported by ARC grant \\ DP120100950.
It started at a MIRIxCanberra workshop sponsored by
the Machine Intelligence Research Institute.
Mayank Daswani and Daniel Filan contributed in the early stages of this paper
and we thank them for interesting discussions and helpful suggestions.
We also thank Nate Soares for useful feedback.

%%%%%%%%%%%%%%%%%%%%%%%%%%%%%%%%%%%%%%%%%%%%%%%%%%%%%%%%%%%%%%%
\bibliographystyle{alpha}
\bibliography{references}

%%%%%%%%%%%%%%%%%%%%%%%%%%%%%%%%%%%%%%%%%%%%%%%%%%%%%%%%%%%%%%%%%%%%
\section*{List of Notation}
\label{app:notation}

\begin{longtable}{lp{0.85\textwidth}}
% Basics
$:=$
	& defined to be equal \\
$\mathbb{N}$
	& the natural numbers, starting with $0$ \\
$\mathbb{R}$
	& the real numbers \\
\iftechreport
$\varepsilon$
	& a small positive real number \\
\fi
% Decision theory
$\A$
	& the (finite) set of possible actions \\
$\E$
	& the (finite) set of possible percepts \\
$\S$
	& the set of hidden states \\
$u$
	& the utility function $u: \E \to [0, 1]$ \\
$a_t$
	& the action in time step $t$ \\
$e_t$
	& the percept in time step $t$ \\
$\ae_{<t}$
	& the first $t - 1$ interactions,
	$a_1 e_1 a_2 e_2 \ldots a_{t-1} e_{t-1}$ \\
$\ae_{i:k}$
	& the interactions between and including time step $i$ and time step $k$,
	$a_i e_i a_{i+1} e_{i+1} \ldots a_k e_k$ \\
$\ae_{1:\infty}$
	& a history of infinite length \\
$s$
	& a hidden state \\
$\pi$
	& a deterministic policy, i.e., a function $\pi: \H \to \A$ \\
$\pi_{t:k}$
	& policy $\pi$ restricted to the time steps between and including $t$ and $k$ \\
$V^{\aevi,\pi}_{\mu,m}$
	& action-evidential value of policy $\pi$ in environment $\mu$
	up to time step $m$, defined in \eqref{eq:SAEDT} \\
$V^{\pevi,\pi}_{\mu,m}$
	& policy-evidential value of policy $\pi$ in environment $\mu$
	up to time step $m$, defined in \eqref{eq:SPEDT} \\
$V^{\cau,\pi}_{\mu,m}$
	& causal value of policy $\pi$ in environment $\mu$
	up to time step $m$, defined in \eqref{eq:SCDT} \\
$k, i$
	& time steps, natural numbers \\
$t$
	& (current) time step \\
$m$
	& lifetime of the agent \\
$P_a$
	& distribution over percepts induced by action $a$ in \ref{eq:SDT} \\
$P$
	& distribution over percepts and actions in one-shot decision making \\
$\mu$
	& an accurate environment model \\
%$\pa_i$
%	& Set of parent nodes of random variable $x_i$ in the causal graph \\
\end{longtable}

%%%%%%%%%%%%%%%%%%%%%%%%%%%%%%%%%%%%%%%%%%%%%%%%%%%%%%%%%%%%%%%
\iftechreport\clearpage
\appendix
\section{Examples}
\label{app:examples}

This section contains the formal calculations for
\autoref{ex:Newcomb}, \autoref{ex:toxoplasmosis},
\autoref{ex:sequential-toxoplasmosis},
\autoref{ex:Newcomb-with-precommitment}, and \autoref{ex:Newcomb-with-looking}.
These calculations are also available as Python code at
\url{http://jan.leike.name/}.

%-------------------------------%
\exampledivider
\begin{example}[Newcomb's Problem]
\label{ex:Newcomb-formal}
%-------------------------------%
This is a formalization of \autoref{ex:Newcomb}.
\begin{itemize}
\item $\S := \{ E, F \}$ where
	$E$ means the opaque box is empty and
	$F$ means the opaque box is full
\item $\A := \{ B_1, B_2 \}$ where
	$B_1$ means one-boxing and $B_2$ means two-boxing
\item $\E := \{ O_0, O_T, O_M, O_{MT} \}$
\item $u(O_0) := 0$, $u(O_T) :=$ 1,000,
	$u(O_M) :=$ 1,000,000, $u(O_{MT}) :=$ 1,001,000
\end{itemize}
Let $\varepsilon > 0$ be a small constant denoting the
accuracy of the predictor.
Because the environment has to assign non-zero probability to all actions,
$\varepsilon$ must be strictly positive.
The environment's distribution $\mu$ is defined as follows.
\begin{align*}
  \mu(E)=\mu(F)           &= 0.5
& \mu(O_T \mid E, B_2)    &= 1 \\
  \mu(B_1 \mid F)=\mu(B_2\mid E)         &= 1 - \varepsilon
& \mu(O_0 \mid E, B_1)    &= 1 \\
  \mu(B_1 \mid E) =\mu(B_2\mid F)        &= \varepsilon
& \mu(O_{MT} \mid F, B_2) &= 1 \\
  &\phantom{=}
& \mu(O_M \mid F, B_1)    &= 1
\end{align*}
By Bayes' rule, 
\[
  \mu(F \mid B_1)
= \frac{\mu(B_1 \mid F) \mu(F)}{\sum_{s\in \S} \mu(B_1\mid s) \mu(s)}
= \frac{\frac{1}{2}(1-\eps)}{\frac{1}{2}(1-\eps)+\frac{1}{2}\eps}
= (1-\eps)
\]
which also gives $\mu(E \mid B_1) = \eps$.
Similarly, $\mu(F \mid B_2) = \eps$
and $\mu(E \mid B_2) = 1 - \eps$.

For EDT we use equation \eqref{eq:EDT}
to compute the value of an action.
Since the percept $e_1$ is generated deterministically,
$\mu(e \mid s, a)$ only attains values $0$ or $1$.
We therefore omit it in the calculation below.
For action $B_1$ we get
\begin{align*}
   V^{\evi,B_1}_{\mu,1}
:= \sum_{e \in \E} \mu(e \mid B_1) u(e)
&= \sum_{e \in \E} \sum_{s \in \S} \mu(e \mid s, B_1) \mu(s \mid B_1) u(e) \\
&= \mu(E \mid B_1) u(O_0) + \mu(F \mid B_1) u(O_M) \\
&= \eps \cdot 0 + (1 - \eps) \cdot 1,000,000
\end{align*}
For action $B_2$ we get
\begin{align*}
   V^{\evi,B_2}_{\mu,1}
:= \sum_{e \in \E} \mu(e \mid B_2) u(e)
&= \sum_{e \in \E} \sum_{s \in \S} \mu(e \mid s, B_2) \mu(s \mid B_2) u(e) \\
&= \mu(E \mid B_2) u(O_T) + \mu(F \mid B_2) u(O_{MT}) \\
&= (1 - \eps) \cdot 1,000 + \eps \cdot 1,001,000 \\
&= 1,000 + \eps \cdot 1,000,000
\end{align*}
For $\eps < 49.95$ (just slightly better than random guessing),
we get that EDT favors $B_1$ over $B_2$:
\[
  V^{\evi,B_1}_{\mu,1}
= (1 - \eps) \cdot 1,000,000
> 500,500
> 1,000 + \eps \cdot 1,000,000
= V^{\evi,B_2}_{\mu,1}
\]

For CDT we use equation \eqref{eq:CDT}
to compute the value of an action.
For action $B_1$ we get
\begin{align*}
   V^{\cau,B_1}_{\mu,1}
:= \sum_{e \in \E} \mu(e \mid \doo(B_1)) u(e)
&= \sum_{e \in \E} \sum_{s \in \S} \mu(e \mid s, B_1) \mu(s) u(e) \\
&= \mu(E) u(O_0) + \mu(F) u(O_M) \\
&= 0.5 \cdot 0 + 0.5 \cdot 1,000,000
 = 500,000
\end{align*}
For action $B_2$ we get
\begin{align*}
   V^{\cau,B_2}_{\mu,1}
:= \sum_{e \in \E} \mu(e \mid \doo(B_2)) u(e)
&= \sum_{e \in \E} \sum_{s \in \S} \mu(e \mid s, B_2) \mu(s) u(e) \\
&= \mu(E) u(O_T) + \mu(F) u(O_{MT}) \\
&= 0.5 \cdot 1,000 + 0.5 \cdot 1,001,000
 = 500,500
\end{align*}
We get that CDT favors $B_2$ over $B_1$
regardless of the prediction accuracy $\eps$:
\[
  V^{\evi,B_1}_{\mu,1}
= 500,000
< 500,500
= V^{\evi,B_2}_{\mu,1}
\]
Moreover, CDT prefers $B_2$ \emph{regardless of the prior over $\mu(E)$}.
Two-boxing is the dominant action
because it yields \$1,000 more regardless of the hidden state.
\end{example}

%-------------------------------%
\exampledivider
\begin{example}[Newcomb with Looking]
\label{ex:Newcomb-with-looking-formal}
%-------------------------------%
This is a formalization of \autoref{ex:Newcomb-with-looking};
it extends \autoref{ex:Newcomb-formal}.

In the first time step, the agent gets to choose between
looking into the box ($L$) and not looking ($N$).
If the agent looks,
the subsequent percept will be $E$ or $F$,
depending on whether the box is empty ($E$) or full ($F$).
If the agent does not look,
the subsequent percept will be $0$.
All three of these percepts $E$, $F$, and $0$ have zero utility.

In the second time step the agent chooses
to one-box ($B_1$) or to two-box ($B_2$).
The payoffs are then based on the boxes' contents
as in \autoref{ex:Newcomb-formal}.
\begin{itemize}
\item $\S := \{ E, F \}$ where
	$E$ means the opaque box is empty and
	$F$ means the opaque box is full
\item $\A := \{ B_1, B_2 \}$ where
	$B_1$ means one-boxing and $B_2$ means two-boxing,
	$L := B_1$ means looking into the box and
	$N := B_2$ means not looking
	(the set of actions has to be the same for all time steps)
\item $\E := \{ E, F, 0, O_0, O_T, O_M, O_{MT} \}$
\item $u(O_0) := 0$, $u(O_T) :=$ 1,000,
	$u(O_M) :=$ 1,000,000, $u(O_{MT}) :=$ 1,001,000,
	$u(E) := u(F) := u(0) := 0$
\end{itemize}
Let $\varepsilon > 0$ be a small constant denoting the
prediction accuracy.
Because the environment has to assign non-zero probability to all actions,
$\varepsilon$ must be strictly positive.
The environment's distribution $\mu$ is defined as follows.
Question marks stand for single actions or percepts whose value is irrelevant.
\begin{align*}
  \mu(E) = \mu(F)         &= 0.5
& \mu(E \mid E, L)        &= 1 \\
  \mu(L \mid F) = \mu(L \mid E) &= 0.5
& \mu(0 \mid E, N)        &= 1 \\
  \mu(N \mid F) = \mu(N \mid E) &= 0.5
& \mu(F \mid F, L)        &= 1 \\
  \mu(B_1 \mid E, ??) &= \eps
& \mu(0 \mid F, N)        &= 1 \\
  \mu(B_1 \mid F, ??) &= 1 - \eps
& \mu(O_0 \mid E, ??B_1) &= 1 \\
  \mu(B_2 \mid E, ??) &= 1 - \eps
& \mu(O_T \mid E, ??B_2) &= 1 \\
  \mu(B_2 \mid F, ??) &= \eps
& \mu(O_M \mid F, ??B_1) &= 1 \\
  &\phantom{=}
& \mu(O_{MT} \mid F, ??B_2) &= 1
\end{align*}
The environment's game tree is given as follows, where dashed lines
connect states indistinguishable by the agent (also known as \emph{information sets}):
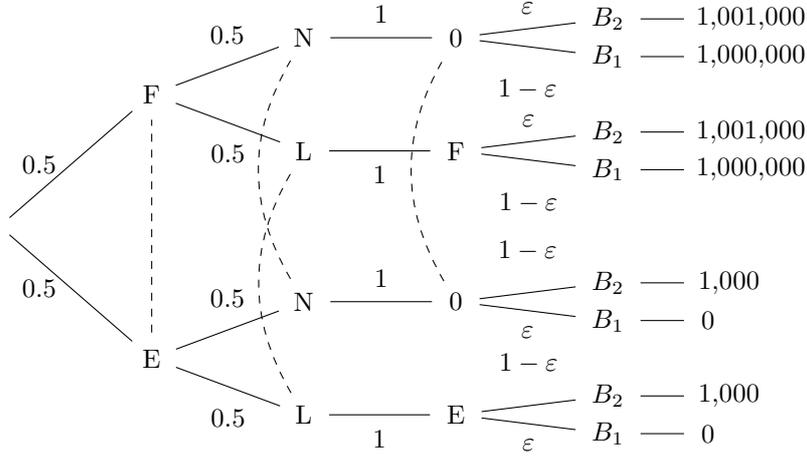
\begin{center}
\begin{tikzpicture}[
	every node/.style = {circle},
	grow = right,  % alignment of characters
	level distance = 20mm,
	level 1/.style = {sibling distance=35mm},
	level 2/.style = {sibling distance=15mm}, 
	level 3/.style = {sibling distance=5mm}, 
	level 5/.style = {level distance = 10mm},
]
\node (init) {}
child {
	node (E) {E}
	child {
		node (EL) {L}
		child {
			node (ELE) {E}
			child {
				node (ELE1) {$B_1$}
				child { node[right] {0} }
			} child {
				node (ELE2) {$B_2$}
				child { node[right] {1,000} }
			}
		}
	} child {
		node (EN) {N}
		child {
			node (EN0) {0}
			child {
				node (EN01) {$B_1$}
				child { node[right] {0} }
			} child {
				node (EN02) {$B_2$}
				child { node[right] {1,000} }
			}
		}
	}
} child {
	node (F) {F}
	child {
		node (FL) {L}
		child {
			node (FLF) {F}
			child {
				node (FLF1) {$B_1$}
				child { node[right] {1,000,000} }
			} child {
				node (FLF2) {$B_2$}
				child { node[right] {1,001,000} }
			}
		}
	} child {
		node (FN) {N}
		child {
			node (FN0) {0}
			child {
				node (FN01) {$B_1$}
				child { node[right] {1,000,000} }
			} child {
				node (FN02) {$B_2$}
				child { node[right] {1,001,000} }
			}
		}
	}
};
\path (init) to node[left] {0.5} (E);
\path (init) to node[left] {0.5} (F);
\path (E) to node[above] {0.5} (EN);
\path (E) to node[below] {0.5} (EL);
\path (F) to node[above] {0.5} (FN);
\path (F) to node[below] {0.5} (FL);
\path (EN) to node[above] {1} (EN0);
\path (EL) to node[below] {1} (ELE);
\path (FN) to node[above] {1} (FN0);
\path (FL) to node[below] {1} (FLF);
\path (EN0) to node[below] {$\eps$}     (EN01);
\path (EN0) to node[above] {$1 - \eps$} (EN02);
\path (ELE) to node[below] {$\eps$}     (ELE1);
\path (ELE) to node[above] {$1 - \eps$} (ELE2);
\path (FN0) to node[below] {$1 - \eps$} (FN01);
\path (FN0) to node[above] {$\eps$}     (FN02);
\path (FLF) to node[below] {$1 - \eps$} (FLF1);
\path (FLF) to node[above] {$\eps$}     (FLF2);

\draw[dashed](E)to(F);
\draw[dashed, bend left](EN)to(FN);
\draw[dashed, bend left](EL)to(FL);
\draw[dashed, bend left](EN0)to(FN0);
\end{tikzpicture}
\end{center}

Using Bayes' rule, we calculate the following conditional probabilities
of the hidden state given a history $a_1$ or $a_1 e_1 a_2$:
\begin{align*}
0.5 &= \mu(E \mid L) = \mu(F \mid L) = \mu(E \mid N) = \mu(F \mid N) \\
1   &= \mu(E \mid LEB_1) = \mu(E \mid LEB_2) = \mu(F \mid LFB_1) = \mu(F \mid LFB_1) \\
\eps &= \mu(E \mid N0B_1) = \mu(F \mid N0B_2) \\
1 - \eps &= \mu(E \mid N0B_2) = \mu(F \mid N0B_1)
\end{align*}

Next, we write out the formula for \ref{eq:SAEDT} for a horizon of $2$
based on \eqref{eq:saedt-iterative-explicit}.
The first percept has no utility, which simplifies the equation.
\[
  V^{\aevi,\pi}_{\mu,2}
= \sum_{e_{1:2}} u(e_2)
    \left( \sum_{s \in \S} \mu(s \mid a_1) \mu(e_1 \mid s, a_1) \right)
    \left( \sum_{s \in \S} \mu(s \mid \ae_1 a_2) \mu(e_2 \mid s, \ae_1 a_2) \right)
\]
where $a_1 = \pi(\epsilon)$ and $a_2 = \pi(\ae_1)$.
The formula for \ref{eq:SPEDT} for a horizon of $2$
based on \eqref{eq:spedt-iterative-explicit} is as follows.
\[
  V^{\pevi,\pi}_{\mu,2}
= \sum_{e_{1:2}} u(e_2)
    \frac{\sum_{s \in \mathcal{S}} \mu(s a_1 e_1 \pi(a_1 e_1))}{\sum_{s \in \mathcal{S}} \sum_{e \in \mathcal{E}} \mu(s a_1 e \pi(a_1 e))}
%    \left( \sum_{s \in \S} \mu(s \mid \pi_{1:2}) \mu(e_1 \mid s, \pi_{1:2}) \right)
    \sum_{s \in \S} \mu(s \mid \ae_1 \pi_2) \mu(e_2 \mid s, \ae_1 a_2)
\]
with $\pi_{1:2}$ and $\pi_2$ defined according to \eqref{eq:policy-cond}.
The formula for \ref{eq:SCDT} for a horizon of $2$
based on \eqref{eq:scdt-iterative-explicit} is as follows.
\[
  V^{\cau,\pi}_{\mu,2}
= \sum_{e_{1:2}} u(e_2)
    \left( \sum_{s \in \S} \mu(s) \mu(e_1 \mid s, a_1) \right)
    \left( \sum_{s \in \S} \mu(s \mid \ae_1) \mu(e_2 \mid s, \ae_1 a_2) \right)
\]
where $a_1 = \pi(\epsilon)$ and $a_2 = \pi(\ae_1)$.

There are six different possible policies:
\begin{itemize}
\item Look and always one-box (curious one-boxer)
\item Look and always two-box (curious two-boxer)
\item Don't look and one-box (incurious one-boxer)
\item Don't look and two-box (incurious two-boxer)
\item Look and one-box iff the box is empty (paradox-lover)
\item Look and one-box iff the box full (fatalistic)
\end{itemize}
Using the formulas above we can calculate their value.
We use $\varepsilon := 0.01$.
\begin{center}
\setlength{\tabcolsep}{2.5mm} % extra space between columns
\begin{tabular}{lccc}
& $V^{\aevi,\pi}_{\mu,2}$ & $V^{\pevi,\pi}_{\mu,2}$ & $V^{\cau,\pi}_{\mu,2}$ \\
Curious one-boxer & 500,000 & \textit{990,000} & 500,000 \\
Curious two-boxer & 501,000 & 11,000 & \textit{501,000} \\
Incurious one-boxer & \textit{990,000} & \textit{990,000} & 500,000 \\
Incurious two-boxer & 11,000 & 11,000 & \textit{501,000} \\
Paradox-lover & 500,500 & 500,500 & 500,500 \\
Fatalistic & 500,500 & 500,500 & 500,500
\end{tabular}
\end{center}
The highest values are displayed in italics.
The incurious one-boxer has the highest action-evidential value.
The curious one-boxer and the incurious one-boxer
have the highest policy-evidential value.
However, of these two policies only the incurious one-boxer
is a time-consistent policy for SPEDT,
because the agent wants to two-box after looking into the box:
\begin{align*}
V^{\aevi,B_1}_{\mu,1}(LF) &= V^{\pevi,B_1}_{\mu,1}(LF) = 1,000,000 \\
V^{\aevi,B_2}_{\mu,1}(LF) &= V^{\pevi,B_2}_{\mu,1}(LF) = 1,001,000 \\
V^{\aevi,B_1}_{\mu,1}(LE) &= V^{\pevi,B_1}_{\mu,1}(LE) = 0 \\
V^{\aevi,B_2}_{\mu,1}(LE) &= V^{\pevi,B_2}_{\mu,1}(LE) = 1,000
\end{align*}
The curious two-boxer and the incurious two-boxer
have the highest causal value,
and they are both time-consistent for SCDT.
\end{example}

%-------------------------------%
\exampledivider
\begin{example}[Newcomb with Precommitment]
\label{ex:Newcomb-with-precommitment-formal}
%-------------------------------%
This is a formalization of \autoref{ex:Newcomb-with-precommitment},
it extends \autoref{ex:Newcomb-formal}.

In the first time step, the agent gets to choose between
signing the contract ($S$) and not signing ($N$).
If the agent signs,
the subsequent percept will be $C$, which costs \$300,000,
and the prediction will be updated to one-boxing.
If the agent does not sign,
the subsequent percept will be $0$ with zero utility.

In the second time step the agent chooses
to one-box ($B_1$) or to two-box ($B_2$).
The payoffs are then based on the boxes' contents
as in \autoref{ex:Newcomb-formal}.
If the agent signed the contract and choses two boxes,
this incurs an additional cost of \$2,000.
\begin{itemize}
\item $\S := \{ E, F \}$ where
	$E$ means the opaque box is empty and
	$F$ means the opaque box is full
\item $\A := \{ B_1, B_2 \}$ where
	$B_1$ means one-boxing and $B_2$ means two-boxing,
	$S := B_1$ means signing the contract and
	$N := B_2$ means not signing
	(the set of actions has to be the same for all time steps)
\item $\E := \{ C, 0, O_0, O_T, O_{-T}, O_M, O_{MT}, O_{M-T} \}$
\item $u(O_0) := 0$, $u(O_T) := 1,000$,
	$u(O_{-T}) := -1,000$
	$u(O_M) := 1,000,000$, $u(O_{MT}) := 1,001,000$,
	$u(O_{M-T}) := 999,000$,
	$u(C) := -300,000$, $u(0) := 0$
\end{itemize}
Let $\varepsilon > 0$ be a small constant denoting the
prediction accuracy.
Because the environment has to assign non-zero probability to all actions,
$\varepsilon$ must be strictly positive.
The environment's distribution $\mu$ is defined as follows.
Question marks stand for single actions or percepts whose value is irrelevant.
\begin{align*}
  \mu(E) = \mu(F)         &= 0.5
& \mu(C \mid E, S)        &= 1 \\
  \mu(S \mid F) = \mu(S \mid E) &= 0.5
& \mu(0 \mid E, N)        &= 1 \\
  \mu(N \mid F) = \mu(N \mid E) &= 0.5
& \mu(C \mid F, S)        &= 1 \\
  \mu(B_1 \mid E, N0) &= \eps
& \mu(0 \mid F, N)        &= 1 \\
  \mu(B_1 \mid F, N0) &= 1 - \eps
& \mu(O_0 \mid E, N0B_1) &= 1 \\
  \mu(B_2 \mid E, N0) &= 1 - \eps
& \mu(O_T \mid E, N0B_2) &= 1 \\
  \mu(B_2 \mid F, N0) &= \eps
& \mu(O_M \mid F, N0B_1) &= 1 \\
  \mu(B_2 \mid\; ?, SC) &= \eps
& \mu(O_{MT} \mid F, N0B_2) &= 1 \\
  \mu(B_1 \mid\; ?, SC) &= 1 - \eps
& \mu(O_M \mid E, SCB_1) &= 1 \\
  &\phantom{=}
& \mu(O_{M-T} \mid E, SCB_2) &= 1
\end{align*}
The environment's game tree is given as follows:
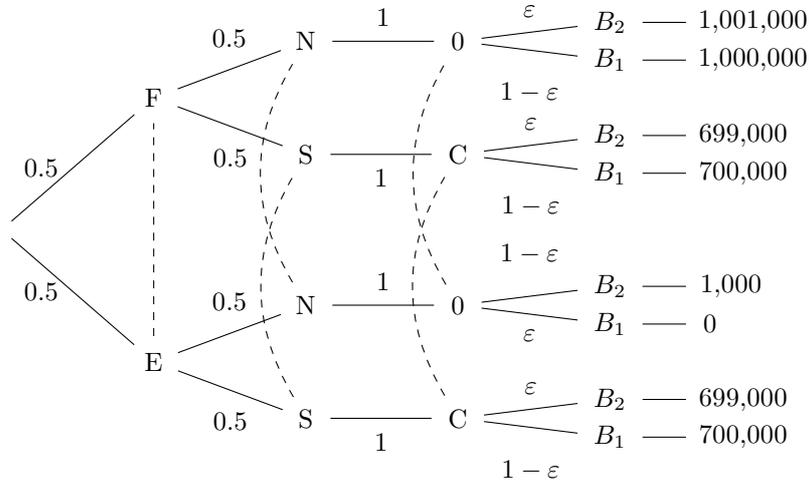
\begin{center}
\begin{tikzpicture}[
	every node/.style = {circle},
	grow = right,  % alignment of characters
	level distance = 20mm,
	level 1/.style = {sibling distance=35mm},
	level 2/.style = {sibling distance=15mm}, 
	level 3/.style = {sibling distance=5mm}, 
	level 5/.style = {level distance = 10mm},
]
\node (init) {}
child {
	node (E) {E}
	child {
		node (ES) {S}
		child {
			node (ESC) {C}
			child {
				node (ESC1) {$B_1$}
				child { node[right] {700,000} }
			} child {
				node (ESC2) {$B_2$}
				child { node[right] {699,000} }
			}
		}
	} child {
		node (EN) {N}
		child {
			node (EN0) {0}
			child {
				node (EN01) {$B_1$}
				child { node[right] {0} }
			} child {
				node (EN02) {$B_2$}
				child { node[right] {1,000} }
			}
		}
	}
} child {
	node (F) {F}
	child {
		node (FS) {S}
		child {
			node (FSC) {C}
			child {
				node (FSC1) {$B_1$}
				child { node[right] {700,000} }
			} child {
				node (FSC2) {$B_2$}
				child { node[right] {699,000} }
			}
		}
	} child {
		node (FN) {N}
		child {
			node (FN0) {0}
			child {
				node (FN01) {$B_1$}
				child { node[right] {1,000,000} }
			} child {
				node (FN02) {$B_2$}
				child { node[right] {1,001,000} }
			}
		}
	}
};
\path (init) to node[left] {0.5} (E);
\path (init) to node[left] {0.5} (F);
\path (E) to node[above] {0.5} (EN);
\path (E) to node[below] {0.5} (ES);
\path (F) to node[above] {0.5} (FN);
\path (F) to node[below] {0.5} (FS);
\path (EN) to node[above] {1} (EN0);
\path (ES) to node[below] {1} (ESC);
\path (FN) to node[above] {1} (FN0);
\path (FS) to node[below] {1} (FSC);
\path (EN0) to node[below] {$\eps$}     (EN01);
\path (EN0) to node[above] {$1 - \eps$} (EN02);
\path (ESC) to node[below] {$1 - \eps$} (ESC1);
\path (ESC) to node[above] {$\eps$}     (ESC2);
\path (FN0) to node[below] {$1 - \eps$} (FN01);
\path (FN0) to node[above] {$\eps$}     (FN02);
\path (FSC) to node[below] {$1 - \eps$} (FSC1);
\path (FSC) to node[above] {$\eps$}     (FSC2);

\draw[dashed](E)to(F);
\draw[dashed, bend left](EN)to(FN);
\draw[dashed, bend left](ES)to(FS);
\draw[dashed, bend left](EN0)to(FN0);
\draw[dashed, bend left](ESC)to(FSC);
\end{tikzpicture}
\end{center}
\end{example}

There are four different possible policies:
\begin{itemize}
\item Sign the contract and one-box (signing one-boxer)
\item Sign the contract and two-box (signing two-boxer)
\item Don't sign the contract and one-box (refusing one-boxer)
\item Don't sign the contract and two-box (refusing two-boxer)
\end{itemize}
Using the formulas from \autoref{ex:Newcomb-with-looking-formal}
we can calculate their value.
We use $\varepsilon := 0.01$.
\begin{center}
\setlength{\tabcolsep}{2.5mm} % extra space between columns
\begin{tabular}{lccc}
& $V^{\aevi,\pi}_{\mu,2}$ & $V^{\pevi,\pi}_{\mu,2}$ & $V^{\cau,\pi}_{\mu,2}$ \\
Signing one-boxer  & 700,000 & 700,00 & \textit{700,000} \\
Signing two-boxer  & 699,000 & 699,000 & 699,000 \\
Refusing one-boxer & \textit{990,000} & \textit{990,000} & 500,000 \\
Refusing two-boxer & 11,000 & 11,000 & 501,000 \\
\end{tabular}
\end{center}
The highest values are displayed in italics.
Both SAEDT and SPEDT refuse the contract:
the refusing one-boxer has the highest action-evidential and
the highest policy-evidential value.
SCDT signs the contract and then one-boxes:
the signing one-boxer has the highest causal value.

%-------------------------------%
\exampledivider
\begin{example}[Toxoplasmosis]
\label{ex:toxoplasmosis-formal}
%-------------------------------%
This is a formalization of \autoref{ex:toxoplasmosis}.
\begin{itemize}
\item $\S := \{ T, H \}$ where
	$T$ means having the toxoplasmosis parasite and
	$H$ means being healthy
\item $\A := \{ P, N \}$ where
	$P$ means petting and $N$ means not petting
\item $\E := \{ P \& T, N \& T, P \& H, N \& H \}$ where
	the percepts just reflect the action and hidden state
\item $u(P \& T) := -9$, $u(N \& T) := -10$,
	$u(P \& H) := 1$, $u(N \& H) := 0$ where
	petting gives a utility of $1$ and
	suffering from the parasite gives a utility of $-10$
\end{itemize}
The environment's distribution $\mu$ is defined as follows.
\begin{align*}
  \mu(T) = \mu(H)       &= 0.5
& \mu(P \& T \mid P, T) &= 1 \\
  \mu(P \mid T)         &= 0.8
& \mu(N \& T \mid N, T) &= 1 \\
  \mu(N \mid T)         &= 0.2
& \mu(P \& H \mid P, H) &= 1 \\
  \mu(P \mid H)         &= 0.2
& \mu(N \& H \mid N, H) &= 1 \\
  \mu(N \mid H)         &= 0.8
\end{align*}
Using Bayes' rule, we calculate the following conditional probabilities.
\begin{align*}
  \mu(T \mid P) &= 0.8
& \mu(H \mid P) &= 0.2
& \mu(T \mid N) &= 0.2
& \mu(H \mid N) &= 0.8
\end{align*}

We consider \ref{eq:EDT} first.
Since the percept $e_1$ is generated deterministically,
$\mu(e \mid s, a)$ only attains values $0$ or $1$.
We therefore omit it in the calculation below.
For action $P$ (petting) we get
\begin{align*}
   V^{\evi,P}_{\mu,1}
:= \sum_{e \in \E} \mu(e \mid P) u(e)
&= \sum_{e \in \E} \sum_{s \in \S} \mu(e \mid s, P) \mu(s \mid P) u(e) \\
&= \mu(T \mid P) u(T \& P) + \mu(H \mid P) u(P \& H) \\
&= 0.8 \cdot (-9) + 0.2 \cdot 1
 = -7
\end{align*}
For action $N$ (not petting) we get
\begin{align*}
   V^{\evi,N}_{\mu,1}
:= \sum_{e \in \E} \mu(e \mid N) u(e)
&= \sum_{e \in \E} \sum_{s \in \S} \mu(e \mid s, N) \mu(s \mid N) u(e) \\
&= \mu(T \mid N) u(T \& N) + \mu(H \mid N) u(H \& N) \\
&= 0.2 \cdot (-10) + 0.8 \cdot 0
 = -2
\end{align*}
Therefore we get that EDT favors $N$ over $P$:
\[
  V^{\evi,P}_{\mu,1}
= -7
< -2
= V^{\evi,N}_{\mu,1}
\]

For \ref{eq:CDT} we get for action $P$ (petting)
\begin{align*}
   V^{\cau,P}_{\mu,1}
:= \sum_{e \in \E} \mu(e \mid \doo(P)) u(e)
&= \sum_{e \in \E} \sum_{s \in \S} \mu(e \mid s, P) \mu(s) u(e) \\
&= \mu(T) u(T \& P) + \mu(N) u(N \& P) \\
&= 0.5 \cdot (-9) + 0.5 \cdot 1
 = -4
\end{align*}
For action $N$ (not petting) we get
\begin{align*}
   V^{\cau,N}_{\mu,1}
:= \sum_{e \in \E} \mu(e \mid \doo(N)) u(e)
&= \sum_{e \in \E} \sum_{s \in \S} \mu(e \mid s, N) \mu(s) u(e) \\
&= \mu(T) u(T \& N) + \mu(H) u(H \& N) \\
&= 0.5 \cdot (-10) + 0.5 \cdot 0
 = -5
\end{align*}
We get that CDT favors $P$ over $N$:
\[
  V^{\evi,P}_{\mu,1}
= -4
> -5
= V^{\evi,N}_{\mu,1}
\]
\end{example}

%-------------------------------%
\exampledivider
\begin{example}[Sequential Toxoplasmosis]
\label{ex:sequential-toxoplasmosis-formal}
We here formalize a version of \autoref{ex:sequential-toxoplasmosis}.
First the agent chooses whether to go to the doctor. Going to the doctor
incurs a fee, but removes the risk of getting sick. 
Agents that do not go to the doctor have a chance of meeting a kitten.
If they meet it, they can choose to pet it or not; infected agents
are more likely to pet the kitten. 
The example is intended to elucidate the difference between SAEDT and
SPEDT, whose decisions we will calculate in detail.
We will not calculate the action of SCDT.

\begin{itemize}
  \item $\S := \{ T$(oxoplasmosis), $H$(ealthy)$\}$.
  \item $\A := \{ Y$(es), $N$(o)$\}$. 
    In this example, an action is taken twice. We use $Y_1$ and $Y_2$,
    and $N_1$ and $N_2$, to distinguish between the first and the second
    action.
  \item $\E := %\{ K, C, H_1, H_P, H_N, T_1,  T_P, T_N, 0 \}
    \{C$(ured), $K$(itten), $S$(ick, not pet kitten), $s$(ick, pet kitten), 
    $P$(et, not sick), 0(neutral)$\}$
  \item  $u(C)=-4$, $u(K):=0$, $u(S):=-10$, $u(s):= -9$, $u(P):=1$, and $u(0)=0$.  
\end{itemize}
The environment's game tree is given as follows, where dashed lines
connect states indistinguishable by the agent.
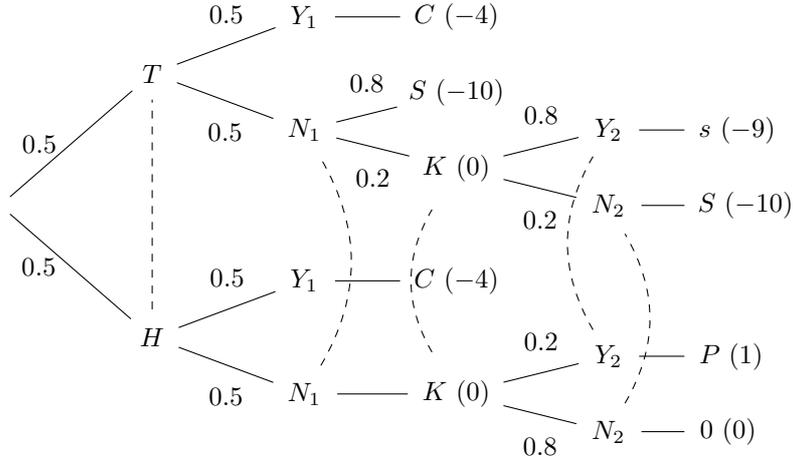
\begin{center}
\begin{tikzpicture}[
	every node/.style = {circle},
	grow = right,  % alignment of characters
	level distance = 20mm,
	level 1/.style = {sibling distance=35mm},
	level 2/.style = {sibling distance=15mm}, 
	level 3/.style = {sibling distance=10mm}, 
	level 5/.style = {level distance = 10mm},
]
\node (init) {}
child {
	node (I) {$H$}
	child {
		node (IE) {$N_1$}
		child {
			node (IEK) {$K$ ($0$)}
                        child {
				node (IEKQ) {$N_2$}
				child { node[right] {$0$ ($0$)} }
			}
			child {
				node (IEKP) {$Y_2$}
				child { node[right] {$P$ ($1$)} }
			}
		}
	} child {
		node (ID) {$Y_1$}
		child {
		 	node (ID0) {$C$ ($-4$)}
		}
	}
}
child {
	node (J) {$T$}
	child {
		node (JE) {$N_1$}
		      child {
			  node (JEK) {$K$ ($0$)}
				child {
                                    node(JEKQ) {$N_2$}
                                    child { node[right] {$S$ ($-10$)} }
                                 }
                                child {
                                    node(JEKP) {$Y_2$}
                                    child { node[right] {$s$ ($-9$)} }
                                    }
			} child {
				node (JES) {$S$ ($-10$)}
			}
	} child {
		node (JD) {$Y_1$}
		child {
			node (JD0) {$C$ ($-4$)}
		}
	}
};
\path (init) to node[left] {0.5} (J);
\path (init) to node[left] {0.5} (I);
\path (J) to node[above] {0.5} (JD);
\path (J) to node[below] {0.5} (JE);
\path (I) to node[above] {0.5} (ID);
\path (I) to node[below] {0.5} (IE);
\path (JE) to node[below] {$0.2$} (JEK);
\path (JE) to node[above] {$0.8$} (JES);
\path (JEK) to node[above] {$0.8$}      (JEKP);
\path (JEK) to node[below] {$0.2$}      (JEKQ);
\path (IEK) to node[above] {$0.2$}      (IEKP);
\path (IEK) to node[below] {$0.8$}      (IEKQ);
\draw[dashed](I)to(J);
\draw[dashed, bend right](IE)to(JE);
\draw[dashed, bend left](IEK)to(JEK);
\draw[dashed, bend left](IEKP)to(JEKP);
\draw[dashed, bend right](IEKQ)to(JEKQ);
\end{tikzpicture}
\end{center}

First, the environment chooses whether to infect the agent or not with
the parasite with probability $0.5$.
The agent then decides whether to see the doctor.
If the agent sees the doctor, this incurs a (utility) fee of $-4$,
but the agent will not be sick.
If the agent does not see the doctor,
there will be a kitten with probability
$0.2$ (or $1$) and the agent will pet it with probability $0.8$ (or $0.2$)
if the parasite is present (or not).
If there is no kitten, the next percept is $S$ or $0$
depending on whether the agent is infected or not.
The agent gets $-10$ utility
if infected and did not see the doctor,
and gets $+1$ utility for petting the kitten.

%\paragraph{Solution.}
We want to compare the choices of SAEDT and SPEDT.
Their two-step value functions are
\[V^{\aevi,\pi}_{\mu, 2} 
  = \sum_{e_1}\mu(e_1\mid a_1)\left(u(e_1) + V^{\aevi,\pi}_{\mu,2}(a_1e_1)\right) \]
\[V^{\pevi,\pi}_{\mu, 2} 
  = \sum_{e_1}\mu(e_1\mid \pi_{1:2})\left(u(e_1) + V^{\pevi,\pi}_{\mu,2}(a_1e_1) \right)\]
where the second step value functions 
\[V^{\aevi,\pi}_{\mu,2}(a_1e_1) = V^{\pevi,\pi}_{\mu,2}(a_1e_1) 
  = \sum_{e_2}\mu(e_2\mid a_1e_1a_2)\cdot u(e_2)\]
are the same for both decision theories.
They only differ by assigning probability $\mu(e_1\mid a_1)$ and $\mu(e_1\mid \pi_{1:2})$
to the first percept, respectively.

Since not petting is always better than petting for evidential agents
(the evidence towards not having the disease weighs stronger than the extra utility),
the only policies that are potentially optimal and time consistent 
are $\pi_1 := N_1 N_2$ and $\pi_2 := Y_1$.

\paragraph{First percept.}
For $\pi_1$ the occurring action-evidential quantities $\mu(e_1\mid a_1)$ are
\begin{align*}
\mu(N_1) 
  &= \sum_{s\in\S}\mu(s, N_1)=\mu(T, N_1) + \mu(H, N_1) 
   = \frac{1}{4} + \frac{1}{4} = \frac{1}{2}\\
\mu(e_1 = S\mid N_1) &= \frac{\sum_{s\in\S}\mu(s, N_1S)}{\mu(N_1)} 
   = \frac{\mu(T, N_1S)}{\mu(N_1)} 
   = \frac{\frac{1}{2}\cdot \frac{1}{2}\cdot \frac{4}{5}}{\frac{1}{2}} 
   = \frac{2}{5}\\
\mu(e_1 = K\mid N_1) &= 1-\mu(S\mid N_1) = \frac{3}{5}
\intertext{and the occurring policy-evidential quantities $\mu(e_1\mid \pi_{1:2})$ are}
\mu(N_1N_2) &= \sum_{s,e_1,e_2} \mu(s, N_1e_1N_2e_2)\\
  &= \mu(T, N_1KN_2S) + \mu(T, N_1SN_20) +\mu(H, N_1KN_20)\\
  &= \frac{1}{100}+\frac{1}{10}+\frac{1}{5}=\frac{31}{100}\\
\mu(e_1 = K\mid N_1N_2) &= \frac{\sum_{s,e_2} \mu(s, N_1KN_2e_2)}{\mu(N_1, N_2)}\\
  &= \frac{\mu(T, N_1KN_2S) + \mu(H, N_1KN_20)}{\mu(N_1N_2)} 
   = \frac{\frac{1}{100}+\frac{1}{5}}{\frac{31}{100}} =\frac{21}{31}\\
\mu(e_1 = S\mid N_1N_2) &= 1-\mu(K\mid N_1N_2) = \frac{20}{31}
\intertext{The policy $\pi_2=\{Y_1\}$ always goes to the doctor for the treatment, and so}
\mu(e_1 = C \mid Y_1) &= 1
\end{align*}
for both AESDT and PESDT.

\paragraph{Second percept.}
With the policy $\pi_2$, the second percept is always empty. 
Under $\pi_1$, the only action
sequence that can reach the second percept is $N_1KN_2$
\begin{align*}
\mu(N_1KN_2) &= \sum_{s}\mu(s, N_1KN_2) = \mu(T, N_1KN_2) + \mu(H, N_1KN_2)\\ 
  &= \frac{1}{100} + \frac{1}{5} = \frac{21}{100}\\ 
\mu(e_2=S\mid N_1KN_2) &= \frac{\sum_{s}\mu(s, N_1KN_2S)}{\mu(N_1KN_2)} 
   = \frac{\mu(T, N_1KN_2S)}{\mu(N_1KN_2)} = \frac{\frac{1}{100}}{\frac{21}{100}} 
   = \frac{1}{21}.
\end{align*}

\paragraph{Value Functions.}
We start by evaluating the recursive definition from the second time step.
The second step value functions are 0 for $\pi_1$ and for the history 
$N_1S$ for $\pi_2$.
For the history $N_1K$, both SAEDT and PAEDT 
assign the following identical value to $\pi_2$:
\begin{align*}
V^{\aevi,\pi_1}_{\mu,2}(N_1K) &= V^{\pevi,\pi}_{\mu,2}(N_1K) 
   = \sum_{e_2}\mu(e_2\mid N_1KN_2)\cdot u(e_2)\\
  &= \mu(e_2=S\mid N_1KN_2)\cdot u(S) + \mu(e_2 = 0\mid N_1KN_2)\cdot u(0)\\
  &= \frac{1}{21}\cdot (-10) + \frac{20}{21}\cdot 0 = -\frac{10}{21} 
\end{align*}
The first step value functions now evaluates to:
\begin{align*}
V^{\aevi,\pi_1}_{\mu, 2} 
  &= \sum_{e_1}\mu(e_1\mid N_1)\cdot \left(u(e_1) + V^{\aevi,\pi_1}_{\mu, 2}(N_1e_1)\right)\\ 
  &= \mu(S\mid N_1) \cdot (u(S) + V^{\aevi,\pi_1}_{\mu,2}(N_1S)) \\
  &\quad~ + \mu(K\mid N_1)\cdot ( u(K) + V^{\aevi,\pi_1}_{\mu,2}(N_1K)) \\
  &= \frac{2}{5}\cdot (-10 + 0) + \frac{3}{5}\cdot ( 0 -\frac{10}{21})
   = -\frac{30}{7}\approx -4.3
\end{align*}

\begin{align*}
V^{\pevi,\pi_1}_{\mu, 2} 
  &= \sum_{e_1}\mu(e_1\mid N_1)\cdot \left(u(e_1) + V^{\pevi,\pi_1}_{\mu, 2}(N_1e_1)\right)\\ 
  &= \mu(S\mid N_1N_2)\cdot (u(S) + V^{\pevi,\pi_1}_{\mu, 2}(N_1S)) \\
  &\quad~ + \mu(K\mid N_1N_2)\cdot ( u(K) + V^{\pevi,\pi_1}_{\mu, 2}(N_1K))\\
  &= \frac{10}{31}\cdot (-10 + 0) + \frac{21}{31}\cdot ( 0 -\frac{10}{21}) 
   = -\frac{110}{31} \approx -3.5
\end{align*}
Meanwhile, the value of $\pi_2$ is 
\begin{align*}
   V^{\aevi,\pi_2}_{\mu, 2}
 = V^{\aevi,\pi_2}_{\mu, 2}
&= \sum_{e_1}\mu(e_1\mid N_1)\left(u(e_1) + V^{\aevi,\pi_2}_{\mu, 2}(N_1e_1)\right)\\
&= \mu(C\mid Y_1)(u(C) + V^{\aevi,\pi_2}_{\mu, 2}(Y_1C)) = 1\cdot(-4 + 0)=-4
\end{align*}
That is, 
$V^{\aevi,\pi_1}_{\mu, 2} < V^{\aevi,\pi_2}_{\mu, 2}=V^{\pevi,\pi_2}_{\mu, 2} < V^{\pevi,\pi_1}_{\mu, 2}$.
So SPEDT but not SAEDT prefers $\pi_1$ to $\pi_2$.
In other words, an SAEDT agent considers himself sufficiently likely to have the
parasite to adopt policy $\pi_2$ of seeing the doctor.
The SPEDT agent relies on the fact that
he would pet the cat in case he saw it, and takes that
as evidence of not being sick. Hence he will instead adopt policy $\pi_1$
of not seeing the doctor.
\end{example}

\fi % \iftechreport

\end{document}